\documentclass[twoside]{article}

\usepackage[preprint]{aistats2027}

\usepackage{amsmath,amssymb,amsfonts}
\usepackage{amsthm}
\usepackage{mathtools}
\usepackage{bm}
\usepackage{graphicx}
\usepackage{subcaption}
\usepackage{booktabs}
\usepackage{microtype}
\usepackage{algorithm}
\usepackage{algpseudocode}
\usepackage{url}
\usepackage{enumitem}
\usepackage[round,authoryear]{natbib}
\usepackage{xcolor}
\usepackage{svg}
\usepackage{appendix}
\usepackage{titletoc}
\usepackage{hyperref}
\usepackage{cleveref}

\newcommand{\DoToC}{%
  \startcontents
  \printcontents{}{1}{%
    \setcounter{tocdepth}{3}%
    \textbf{Contents}\par\vskip3pt\hrule\vskip5pt
  }%
  \vskip3pt\hrule\vskip5pt
}

\newtheorem{proposition}{Proposition}

\newtheorem{example}{Example}
\theoremstyle{definition}
\newtheorem{definition}{Definition}
\theoremstyle{remark}

\newcommand{\R}{\mathbb{R}}
\newcommand{\E}{\mathbb{E}}

\newcommand{\grad}{\nabla}

\newcommand{\gr}[1]{\textcolor{black}{#1}}

\newcommand{\mrk}[1]{\textcolor{black}{#1}}
\newcommand{\gdr}[1]{\textcolor{black}{#1}}

\newcommand{\gddr}[1]{\textcolor{black}{#1}}

\begin{document}

\runningtitle{TASER: Geometry-Driven Smoothness}
\runningauthor{Anonymous Authors}

\twocolumn[
\aistatstitle{TASER: Task-Aware Stein Regularisation for Geometry-Driven Smoothness}

\aistatsauthor{
Micha\l{} Kozyra \And Gesine Reinert
}

\aistatsaddress{
University of Oxford \\
\texttt{michal.kozyra@seh.ox.ac.uk}
\And
University of Oxford \\
\texttt{reinert@stats.ox.ac.uk}
}
]

\begin{abstract}
Deep neural networks remain sensitive to distribution shifts,
motivating regularisation informed by data geometry.
We introduce \textbf{TASER (Task-Aware Stein Regularisation)},
a compact training-time framework that couples model
derivatives to density scores by penalising the Langevin
Stein operator.
We establish a weighted Sobolev
interpretation and a hierarchy of stability bounds under
progressively stronger distributional assumptions. \gddr{Moreover we introduce  TASERLight, which is based on a first-order Stein operator.}
TASER \gddr{and TASERLight} require neither explicit tangent-space estimation
nor an inner perturbation search.
Controlled synthetic examples illustrate 
\gddr{their}
effects on
extrapolation and directional sensitivity.
Experiments on spectroscopy, skeleton motion, ECG, MNIST,
and CIFAR-10 demonstrate gains under distribution shift
and complementary benefits when combined with isotropic
regularisation.
\end{abstract}

\section{Introduction}

Deep neural networks achieve strong in-distribution performance, yet remain fragile under distribution shift and adversarial perturbations \citep{hendrycks2019benchmarking, goodfellow2015explaining, madry2018towards}. One contributing factor is \emph{misaligned input sensitivity}: the model exhibits large responses to small, task-irrelevant perturbations, while potentially underreacting to semantically meaningful variations \citep{tsipras2018robustness}. Under distribution shift, such sensitivity can degrade generalisation when test inputs deviate from the training distribution in structured ways \citep{hendrycks2019benchmarking}. Adversarial perturbations provide a related example, exploiting directions in input space along which small changes induce large changes in model output \citep{goodfellow2015explaining, madry2018towards}.

A large body of work addresses this issue by regularising model sensitivity. Classical approaches include weight decay and spectral constraints \citep{loshchilov2019decoupled, miyato2018spectral}, while more direct methods penalise gradients or enforce Lipschitz bounds \citep{jakubovitz2018improving, cisse2017parseval}. Adversarial training further seeks robustness by optimising against worst-case perturbations \citep{madry2018towards, zhang2019theoretically}. However, controlling overall sensitivity does not by itself determine which input directions should be regularised most strongly. Constraints based on parameter norms or ambient distances need not reflect the \emph{geometry of the training distribution}. Consequently, regularisation may suppress sensitivity to task-relevant variations while insufficiently controlling sensitivity in directions that move inputs away from high-probability regions.

We regularise model behaviour using the \emph{geometry of the training
distribution}. Our starting point is Stein's method, whose operators
satisfy identities of the form $\mathbb{E}_p[\mathcal{T}_p f]=0$
for suitable functions $f$ \citep{stein1972bound, ley2017stein}.
For a density $p$ with score $s_p=\nabla\log p$, the Langevin Stein operator
\begin{equation}\label{eq:stein-operator}
\mathcal{L}_p f(x)=\Delta f(x)+s_p(x)^\top\nabla f(x)
\end{equation}
couples model curvature with score-aligned sensitivity.
Here $\Delta$ is the Laplacian\gddr{, $\nabla$ is the gradient}, and $f$ is a scalar predictive
probe, such as a regression output or a classification margin.
The score supplies information about the training distribution,
while the derivatives of the probe describe the
local behaviour of a model.

We propose \textbf{TASER (Task-Aware Stein Regularisation)}, a training-time regularisation framework that penalises pointwise Stein residuals:
\begin{equation}
\mathcal{L}_{\mathrm{total}}(\theta)
= \mathcal{L}_{\mathrm{task}}(\theta) \\
 + \lambda\,\mathbb{E}_{X\sim p}\big[(\mathcal{L}_p f_\theta(X))^2\big].
\end{equation}

Unlike conventional regularisers that act uniformly across input space, TASER imposes constraints that are explicitly shaped by the distribution $p$. In particular, the score-weighted term $s_p(x)^\top \nabla f(x)$ \gdr{in $\mathcal{L}_p$} measures sensitivity along directions in which the data density changes most rapidly, while the Laplacian term captures probe curvature. TASER couples these quantities in a distribution-aware residual. 
Together, they enforce a form of \emph{geometry-aware smoothness} that aligns model sensitivity with the structure of the data.

We also develop \textbf{TASERLight}, \gddr{based on a first-order Stein operator}  
that omits the pointwise Laplacian and centres the remaining
score-directed response under the training distribution.
It focuses on directional sensitivity and requires only
first derivatives of the probe.
Full TASER and TASERLight therefore use the same score
information but differ in their treatment of probe curvature.
Both formulations and their relationship are developed
in Section~\ref{sec:taser}.

\begin{figure}[t]
  \centering
  \includegraphics[width=\linewidth]
  {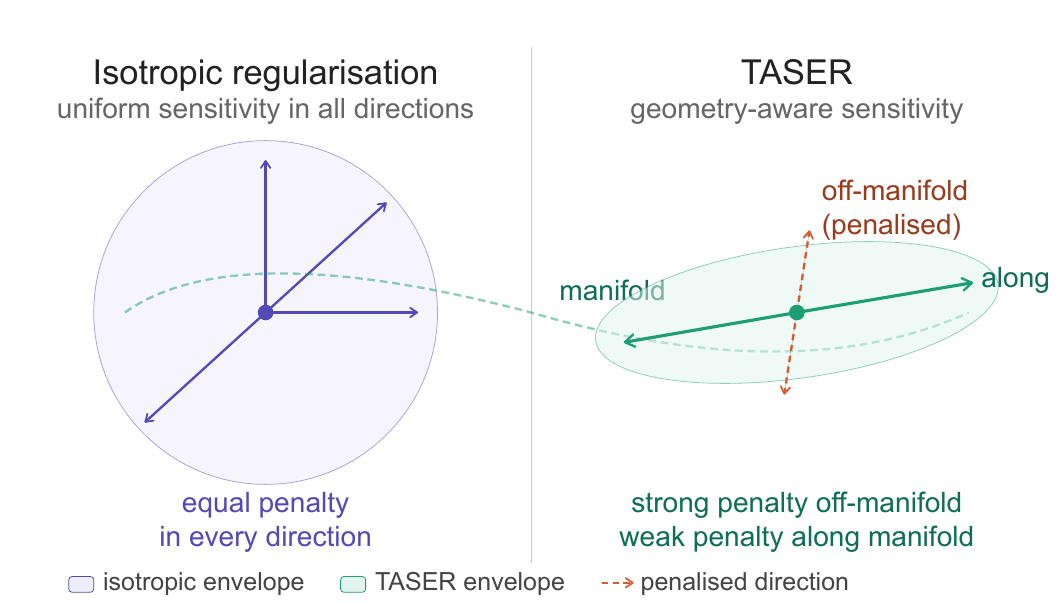}
  \caption{\footnotesize
  \textbf{Isotropic versus distribution-aware smoothness.}
  Isotropic gradient regularisation weights sensitivity equally
  across input directions. The score-directed component shared
  by TASER and TASERLight targets the local density-score direction,
  which is approximately normal to the manifold in the illustrated
  geometry.}
  \label{fig:my_figure}
\end{figure}

\gddr{The geometric perspective} 
is particularly natural in high-dimensional
settings where data concentrate near lower-dimensional structures
\citep{fefferman2016testing}.
When density changes more rapidly across these structures than
along them, the score is predominantly normal to the data manifold.
Its inner product with the probe gradient then mainly measures
sensitivity to normal displacement, providing a way to target
off-manifold sensitivity without explicitly estimating the manifold
\citep{fawzi2018analysis}.
TASERLight regularises variation in this directional response,
while full TASER couples it to probe curvature.
The two variants thus use the same distributional information
but impose different smoothness preferences.

Our analysis connects both variants to distribution-dependent
smoothness and stability under distribution shift.
We characterise how each penalty constrains probe derivatives
and show that both control changes in the average operator
response when the input distribution shifts.
Under further assumptions on the distribution, the full penalty
also controls changes in the average predictive probe.
For TASERLight, the corresponding bound additionally depends
on fluctuations of the probe's Laplacian.
The analysis thus explains both the common motivation
and the different forms of control offered by the two variants.

TASER and TASERLight are simple to implement and broadly applicable.
Scores can be obtained from fitted density models, dedicated
score-matching networks, or compatible diffusion checkpoints
\citep{hyvarinen2005estimation, ho2020denoising, song2021score}.
Neither variant requires explicit tangent-space estimation
or a search for adversarial perturbations, and the score
estimator is unnecessary at inference.
The framework is agnostic to architecture and task, and can
be combined with existing regularisers and training pipelines.

\paragraph{Contributions.}
\begin{itemize}[nosep, noitemsep,leftmargin=*]
\item We introduce full and first-order Stein regularisers
for incorporating training-distribution information into
model optimisation through task-dependent probes.

\item We characterise the smoothness encouraged by both variants
and establish stability bounds under distribution shift,
clarifying the role of the score and the different treatment
of probe curvature.

\item We demonstrate directional sensitivity control in regression
and sphere examples, and evaluate clean and shifted performance
on spectroscopy, skeleton, ECG, MNIST, and CIFAR-10 tasks.
The results support benefits in several settings and useful
composition with isotropic regularisation.
\end{itemize}

More broadly, TASER and TASERLight use Stein operators directly
as tools for \emph{training}, connecting density estimation
with the behaviour of predictive models.

\section{Related Work}

\paragraph{Regularising model sensitivity.}
Weight decay and spectral constraints control model sensitivity through
parameter norms \citep{loshchilov2019decoupled,miyato2018spectral,cisse2017parseval},
while Jacobian penalties directly regularise input derivatives
\citep{jakubovitz2018improving}. These approaches encourage smoothness
without explicitly representing the training density, leaving open how
to distinguish task-relevant variation from directions that move inputs
away from high-probability regions.

\paragraph{Consistency and manifold regularisation.}
Manifold regularisation and manifold tangent classifiers use data
geometry to guide smoothness \citep{belkin2006manifold,rifai2011manifold}.
Virtual adversarial training enforces local prediction consistency
\citep{miyato2018virtual}, while tangent-normal adversarial regularisation
separates tangent and normal perturbations using an estimated manifold
\citep{yu2019tangent}. TASER and TASERLight share this motivation but
use density scores, avoiding explicit tangent-space estimation and an
inner perturbation search. Compatible score models trained for generation
or denoising can therefore be reused across prediction tasks
\citep{song2021score}.

\paragraph{Score estimation.}
Scores can be obtained analytically from fitted Gaussian or mixture
densities, or learned through score matching without evaluating
normalising constants \citep{hyvarinen2005estimation}. Denoising score
matching connects score estimation to noise-corrupted observations
\citep{vincent2011connection}, underpinning diffusion and score-based
generative models that learn scores across noise levels
\citep{ho2020denoising,song2021score}. These estimators provide complementary
ways to supply density information to our regulariser, from simple
parametric models to flexible neural scores.

\paragraph{Stein's method in machine learning.}
Stein operators underpin goodness-of-fit tests, sample-quality assessment
and discrepancy measures; see \citet{liu2026probabilistic} for a survey.
Task-aware Stein diagnostics also detect distribution shift
\citep{kozyra2026tastetaskawareoutofdistributiondetection}.
These diagnostic applications motivate using distribution-sensitive
operator responses to shape predictions. TASER and TASERLight turn
such responses into training objectives, coupling the estimated density
to a task-dependent predictive probe.

\section{Methods: Task-Aware Stein Regularisation}
\label{sec:taser}

\subsection{Problem setting and the density operator}

Consider a supervised learning problem with inputs
$x\in\mathbb R^d$, drawn from a distribution with density $p$,
and targets $y$.
We regularise a scalar probe $f_\theta:\mathbb R^d\to\mathbb R$
derived from the model.
The task loss fits the targets, while our regulariser uses
information about $p$ to shape the input sensitivity of the probe. For regression, the probe may be the model output itself;
for classification, it may be a scalar function of the logits
or softmax probabilities.

Our construction starts from the Stein density operator.
For a smooth positive density $p$ and a differentiable
vector-valued test function $v:\mathbb R^d\to\mathbb R^d$,
define
\begin{equation}\label{eq:stein-densityoperator}
\begin{aligned}
\mathcal A_p v(x)
&=\frac{1}{p(x)}\nabla\cdot \!\bigl(p(x)v(x)\bigr)\\
&=\nabla\cdot v(x)+s_p(x)^\top v(x),
\end{aligned}
\end{equation}
where $s_p=\nabla\log p$.
When the relevant quantities are integrable and the boundary
flux vanishes, the divergence theorem gives $\mathbb E_p[\mathcal A_p v]=0$.
We assume the regularity and boundary conditions required
for the identities below; further details are given in
Appendix~\ref{app:proofs}.

The density operator provides a common starting point for
combining distributional information with predictive behaviour.
The density determines $s_p$, while the choice of $v$
determines how the model enters the operator.
We construct two scalar operators by choosing
test functions built from the model.

\subsection{TASER and TASERLight}

\paragraph{Full TASER.}
Choosing $v=\nabla f$ \gddr{in \eqref{eq:stein-densityoperator}} gives
\begin{equation}\label{eq:langevin-construction}
\begin{aligned}
\mathcal A_p(\nabla f)
&=\Delta f+s_p^\top\nabla f
=\mathcal L_p f,
\end{aligned}
\end{equation}
recovering the Langevin Stein operator
in~\eqref{eq:stein-operator}.
Its two terms couple probe curvature with the directional
derivative along the score. Equivalently,
$\mathcal L_p f=p^{-1}\operatorname{div}(p\nabla f)$:
the operator measures the divergence of the
density-weighted probe gradient.
Under the stated conditions,
$\mathbb E_p[\mathcal L_p f]=0$.

\paragraph{TASERLight.}
\gddr{Choosing}
$v=f\,s_p$ \gddr{in \eqref{eq:stein-densityoperator}} 
gives
\begin{equation}\label{eq:score-test-function}
\mathcal A_p(f\,s_p)
=s_p^\top\nabla f+f\bigl(\nabla \cdot s_p+\|s_p\|^2\bigr).
\end{equation}
The first term measures score-directed sensitivity.
The second depends on the value of the probe itself, and, unlike the Langevin operator, it generally changes when a constant is added
to $f$. 
To obtain a sensitivity penalty that is invariant to such
additive offsets, we replace the second term by its expectation
under $p$. This defines
\begin{equation}\label{eq:stein-operator-light}
\mathcal B_p f(x)
=s_p(x)^\top\nabla f(x)
+\mathbb E_p\!\left[
 f\bigl(\nabla \cdot s_p+\|s_p\|^2\bigr)
 \right].
\end{equation}
Since $\mathbb E_p[\mathcal A_p(f\,s_p)]=0$, the expectation
in~\eqref{eq:stein-operator-light} equals
$-\mathbb E_p[s_p^\top\nabla f]$. Hence
\begin{equation}\label{eq:first-order}
\mathcal B_p f(x)
=s_p(x)^\top\nabla f(x)
-\mathbb E_p[s_p^\top\nabla f].
\end{equation}
Thus $\mathcal B_p$ is the centred score-directed derivative,
with $\mathbb E_p[\mathcal B_p f]=0$ and
$\mathcal B_p(f+c)=\mathcal B_p f$ for any constant $c$.
Although its derivation involves derivatives of the score,
the equivalent expression~\eqref{eq:first-order} requires
only score values and first derivatives of the probe.
This is the operator used by TASERLight. 

The relationship between the two operators follows from
$\mathbb E_p[\mathcal L_p f]=0$, which gives
$\mathbb E_p[s_p^\top\nabla f]=-\mathbb E_p[\Delta f]$.
Consequently, \gddr{as shown in Appendix \ref{app:taserlight},}
\begin{equation}\label{eq:lpconnection}
\mathcal L_p f
=\mathcal B_p f+\Delta f-\mathbb E_p[\Delta f].
\end{equation}
The operators therefore differ by fluctuations of the
Laplacian. TASERLight retains the score-directed response,
while full TASER also incorporates this varying curvature
contribution.

\paragraph{Training objective.}
We call $\mathcal L_p f(x)$ and $\mathcal B_p f(x)$ the
respective pointwise Stein residuals. Their squared expectations
define the penalties
\begin{equation}\label{eq:taser-penalties}
\begin{aligned}
\mathcal R_{\mathcal L}(f)
&=\mathbb E_p[(\mathcal L_p f)^2], \qquad 
\mathcal R_{\mathcal B}(f)
=\mathbb E_p[(\mathcal B_p f)^2].
\end{aligned}
\end{equation}
Full TASER uses $\gddr{\mathcal R=} \mathcal R_{\mathcal L}$ and TASERLight
uses $\gddr{\mathcal R=}\mathcal R_{\mathcal B}$ in the objective
\begin{equation}\label{eq:taser-objective}
\mathcal L_{\mathrm{total}}(\theta)
=\mathcal L_{\mathrm{task}}(\theta)
+\lambda\,\mathcal R(f_\theta).
\end{equation}
For the exact score and admissible probes, both operators have
zero expectation under $p$, so these penalties are variances.

\paragraph{Estimated scores and centring.} In practice, $s_p$ is replaced by an estimate $\widehat s$ whose errors can introduce a nonzero
mean into the operators. We therefore add
empirical centring:
\begin{equation}\label{eq:taser-empirical}
\widehat{\mathcal R}_{\mathcal A}(f)
\!=\!\frac{1}{m}\sum_{i=1}^{m}
  \bigl(\widehat r_{\mathcal A,i}
        -\overline r_{\mathcal A}\bigr)^2,
\, \, \, \, 
\overline r_{\mathcal A}
\!=\!\frac{1}{m}\sum_{i=1}^{m}\widehat r_{\mathcal A,i}.
\end{equation}
Here $\widehat r_{\mathcal L,i}$ is the operator evaluated using the estimated score $\widehat s_p$, while
$\widehat r_{\mathcal B,i}
=\widehat s_p(x_i)^\top\nabla f(x_i)$ is the raw first-order
response, whose centring implements $\mathcal B_p$ empirically.
Centring removes the empirical mean offset, but does not
eliminate input-dependent score errors.

\paragraph{Geometric interpretation.}
Both operators contain $s_p^\top\nabla f$, which measures
the probe's response along the score direction.
Near a data manifold, decompose the probe gradient into
tangent and normal components:
\[
\nabla f(x)
=\Pi_T(x)\nabla f(x)+\Pi_N(x)\nabla f(x).
\]
When density varies predominantly in normal directions,
the score is approximately normal and
\[
s_p(x)^\top\nabla f(x)
\approx s_p(x)^\top\Pi_N(x)\nabla f(x).
\]
The directional term then predominantly measures sensitivity
to departures from the manifold.
Figure~\ref{fig:my_figure} illustrates the shared
directional interpretation.

This control is local. At a fixed input $x$,
\[
(s_p(x)^\top\nabla f(x))^2
=\nabla f(x)^\top(s_p(x)s_p(x)^\top)\nabla f(x),
\]
where $s_ps_p^\top$ has rank at most one.
The term therefore probes one local direction, with different
directions encountered as the score varies across inputs.
The resulting effect on normal and tangential sensitivity
depends on the density geometry and the model.

\paragraph{Implementation.}
Full TASER estimates the Laplacian using Hutchinson's trace
estimator and Hessian--vector products
\citep{hutchinson1989stochastic}.
TASERLight \gddr{requires} 
only
probe input gradients to evaluate its residual.
Algorithm~\ref{alg:taser} switches between the variants through
a Boolean \texttt{light} flag.

\begin{algorithm}[t]
\caption{Task-Aware Stein Regularisation 
}
\label{alg:taser}
\begin{algorithmic}[1]
\Require Training set $\mathcal D$, model output (or probe) $f_\theta$,
score $\tilde s(x)$, weight $\lambda$, number of Hutchinson probes $K$,
boolean \texttt{light}
\For{each training step}
    \State Sample minibatch $\{(x_i,y_i)\}_{i=1}^B$
    \State Compute task loss $\mathcal L_{\mathrm{task}}$
    \For{each $x_i$}
        \State Compute gradient $\nabla f_\theta(x_i)$
        \State $r_i \gets \tilde s(x_i)^\top \nabla f_\theta(x_i)$
        \If{\textbf{not} \texttt{light}}
            \State Estimate $\Delta f_\theta(x_i)$ using $K$ Hutchinson probes
            \State $r_i \gets r_i + \widehat{\Delta f_\theta}(x_i)$
        \EndIf
    \EndFor
    \State $\bar r \gets \frac{1}{B}\sum_i r_i$
    \State $\mathcal R \gets \frac{1}{B}\sum_i (r_i-\bar r)^2$
    \State Update $\theta$ using
           $\mathcal L_{\mathrm{task}}+\lambda\mathcal R$
\EndFor
\end{algorithmic}
\end{algorithm}

\paragraph{Probe choice.}
The probe should be scalar-valued, reflect task-relevant behaviour,
and be differentiable in the inputs and model parameters;
full TASER additionally requires second input derivatives.
Theoretical results assume the stated 
integrability and boundary
conditions. For regression, the natural choice is the model output.
For classification, our default is the softmax probability assigned
to the predicted class, unless specified otherwise.
Additional robustness evaluations use the smooth logit-margin probe.
Experimental choices and implementation details are deferred
to the Appendix.

\section{Theoretical Analysis}

\paragraph{Weighted Sobolev formulation.}
The TASER penalty admits a Sobolev-type interpretation.
Under our regularity and integration-by-parts assumptions,
\begin{eqnarray}
\label{eq:sobolev-stein-identity}
\mathbb{E}_{p}\!\left[(\mathcal{L}_p f(X))^2\right]
&={}& \mathbb{E}_{p}\!\left[\|\nabla^2 f(X)\|_F^2\right] \\
&&+ \mathbb{E}_{p}\!\left[
\nabla f(X)^\top H_p(X)\nabla f(X)
\right],
\nonumber 
\end{eqnarray}
where $H_p(x)=-\nabla^2\log p(x)$ and $\|\cdot\|_F$
denotes the Frobenius norm; see Appendix~\ref{app:proofs}. \eqref{eq:sobolev-stein-identity}
%
expresses TASER as a nonnegative,
distribution-dependent quadratic form involving first-
and second-order derivatives of the probe. The Hessian
term measures probe curvature, while the gradient term
is weighted by the local curvature of the negative
log-density.

The interpretation of these contributions depends on $p$.
For a log-concave density, $H_p(x)\succeq 0$, so the identity
decomposes TASER into nonnegative curvature and anisotropic
gradient penalties. Unlike isotropic Sobolev regularisation,
the gradient penalty then weights sensitivity according to
the geometry of the density. For a general density,
$H_p(x)$ may be indefinite and the gradient contribution
may be negative; the identity still holds, but nonnegativity
of the total penalty does not imply separate control of
the two contributions.

Under the stronger condition $H_p(x)\succeq\kappa I$
for some $\kappa>0$, the identity yields
\[
\mathcal{R}_L(f)
\geq
\mathbb{E}_p[\|\nabla^2 f\|_F^2]
+\kappa\,\mathbb{E}_p[\|\nabla f\|^2],
\]
giving explicit control of both first- and second-order
Sobolev seminorms. In particular, for the standard Gaussian,
$H_p(x)=I$, and TASER equals the sum of the isotropic
gradient and Hessian penalties under Gaussian weighting.


\paragraph{Stability under distribution shift.}
We next interpret TASER and TASERLight as controlling the response
of the model under shifted or adversarial input distributions.
Let $q$ denote the distribution of perturbed inputs, for example
the distribution induced by applying an attack mechanism to samples
from $p$. Assume that $q$ is absolutely continuous with respect to $p$.
Denote the $\chi^2$ divergence between $q$ and $p$ by
$
\chi^2(q\|p)
=
\mathbb{E}_{p}\!\left[
\left(\frac{q(X)}{p(X)}-1\right)^2
\right].
$
Letting
$
\mathcal{Q}_\rho
=
\left\{
q : \chi^2(q\|p) \le \rho^2
\right\},
$
for either $\mathcal A_p\in\{\mathcal L_p,\mathcal B_p\}$
with a square-integrable residual, we show in
Appendix~\ref{app:proofs} that
\begin{equation}
\label{eq:sup-shift-bound}
\sup_{q\in \mathcal{Q}_\rho}
\left|
\mathbb{E}_{q}[\mathcal A_p f(X)]
\right|
\le
\rho
\sqrt{
\mathbb{E}_{p}\!\left[
(\mathcal A_p f(X))^2
\right]
}.
\end{equation}
Thus, minimising either penalty controls the worst-case Stein
response over a neighbourhood of the training distribution.
For the full operator, the left-hand side of
\eqref{eq:sup-shift-bound} has a geometric interpretation
through the identity
\begin{equation}
\label{eq:projection-identity-shift}
\mathbb{E}_{q}[\mathcal{L}_p f(X)]
=
-\mathbb{E}_{q}\left[
\nabla f(X)^\top \nabla \log \frac{q(X)}{p(X)}
\right],
\end{equation}
(see \citep{kozyra2026tastetaskawareoutofdistributiondetection}),
which holds under the corresponding integration-by-parts
assumptions for $q$.
The term $\nabla\log(q/p)$ describes the local direction
of the distributional shift from $p$ to $q$.
Therefore, the full Stein functional measures the average
alignment between the model sensitivity $\nabla f$ and
the shift direction.
Bound~\eqref{eq:sup-shift-bound} shows that full TASER suppresses
this alignment uniformly over all distributions in a $\chi^2$
neighbourhood of $p$.
For TASERLight, the projection identity includes an additional
centred-Laplacian term; see Appendix~\ref{app:proofs}.

\gddr{We highlight} that our construction does not constitute
a formal bound on model error. Instead, it controls a task-aware
sensitivity functional associated with the shifted or attack-induced
distribution. 
\gddr{B}oth variants provide an attack-agnostic
stability guarantee: no admissible shifted distribution can induce
a large Stein response unless either the shift is far from the
training distribution or the corresponding penalty itself is large.

\paragraph{Poincar\'e bound on expected probe shifts.}
Suppose additionally that $p$ satisfies
$\operatorname{Var}_p(g)
\leq \kappa^{-1}\mathbb{E}_p[\|\nabla g\|^2]$
for all admissible $g$, with spectral gap $\kappa>0$.
Under 
\gddr{our} integration-by-parts assumptions 
and assuming
the terms below are finite, we obtain
\begin{align*}
\sup_{q\in\mathcal Q_\rho}
\left|\mathbb E_q[f]-\mathbb E_p[f]\right|
&\leq \frac{\rho}{\kappa}
       \sqrt{\mathcal R_{\mathcal L}(f)}
\\
&\leq \frac{\rho}{\kappa}
\left(
\sqrt{\mathcal R_{\mathcal B}(f)}
+
\sqrt{\operatorname{Var}_p(\Delta f)}
\right). \nonumber 
\end{align*}
The second inequality follows from \gddr{\eqref{eq:lpconnection}}; 
\gddr{see} Appendix~\ref{app:proofs}.
Controlling the full TASER penalty limits the change in the
expected probe over a $\chi^2$ neighbourhood of the training
distribution, with the strength of the bound depending on
the spectral gap.
For TASERLight, this control \gddr{also} 
depends on the
variability of the Laplacian, which the variant does not explicitly penalise.
The density, spectral-gap, and shift assumptions may not hold
for real datasets and their shifts, so this result remains
a conditional population-level guarantee.

\section{Experimental Results}
\label{sec:experiments}
\subsection{Simulated data}

\paragraph{1D regression: extrapolation under distribution shift.}

We train a small fully connected network on $y=\sin(x)$,
$x\sim\mathcal{N}(0,1)$, using either $\ell_2$ weight decay or TASER.
Both methods fit the training region, but TASER produces more stable
extrapolations across regularisation strengths
(Figure~\ref{fig:toy_forecasts}). It achieves lower ID MSE at every
tested nonzero $\lambda$ and lower wide-grid MSE at three of four
strengths (Table~\ref{tab:1d-res}).

\begin{figure}[ht]
\centering
\includegraphics[width=0.90\linewidth]{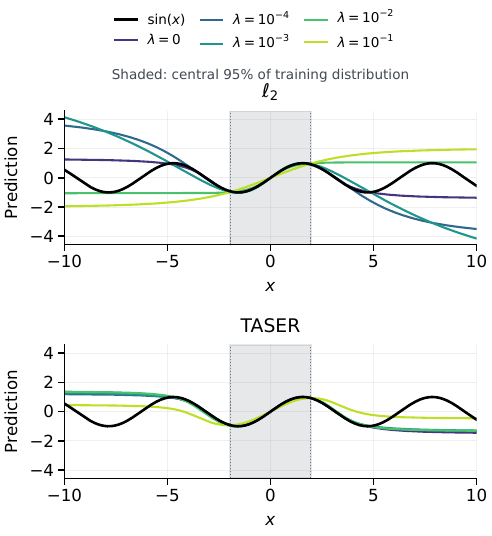}
\caption{Extrapolation with $\ell_2$ weight decay (top) and TASER
(bottom), across regularisation strengths. Shading marks the central
$95\%$ of the $\mathcal{N}(0,1)$ training distribution.}
\label{fig:toy_forecasts}
\end{figure}

\begin{table}[ht]
\centering
\caption{Test MSE for 1D regression: in-distribution (ID) inputs
and a uniform OOD grid over $[-10,10]$. The unregularised baseline
appears once, at $\lambda=0$.}
\label{tab:toy_mse}
\small
\setlength{\tabcolsep}{3pt}
\begin{tabular}{c cc cc}
\toprule
& \multicolumn{2}{c}{$\ell_2$} & \multicolumn{2}{c}{TASER} \\
\cmidrule(lr){2-3}\cmidrule(lr){4-5}
$\lambda$ & ID & OOD & ID & OOD \\
\midrule
$0$ & $4.19\!\times\!10^{-5}$ & 1.267 & -- & -- \\
$10^{-4}$ & $2.37\!\times\!10^{-4}$ & 5.335 & $1.85\!\times\!10^{-6}$ & 1.225 \\
$10^{-3}$ & $1.80\!\times\!10^{-3}$ & 5.407 & $2.68\!\times\!10^{-5}$ & 1.245 \\
$10^{-2}$ & $1.29\!\times\!10^{-2}$ & 1.158 & $5.95\!\times\!10^{-4}$ & 1.336 \\
$10^{-1}$ & $6.00\!\times\!10^{-2}$ & 2.546 & $2.04\!\times\!10^{-2}$ & 0.541 \\
\bottomrule
\end{tabular}
\label{tab:1d-res}
\end{table}

\paragraph{Hypersphere: directional smoothness.}
We next study how regularisation affects sensitivity around a curved
data manifold. Inputs $x\in\mathbb{R}^{2n}$ are sampled from a noisy
neighbourhood of $S^{n-1}\times\{0\}$, with regression target
$y=\sum_{i=1}^{n}\sin(x_i)$. We compare MLE, TASERLight using the
analytic density score, and isotropic gradient regularisation.
We measure mean gradient norms projected onto the geometric normal
and tangent spaces at fixed normal-block displacement $\|b\|=0.5$,
varying dimension and regularisation strength
(Figure~\ref{fig:toy_sphere_geometry}).
For $n=40$ and $\lambda=10^{-3}$, TASERLight reduces the normal
gradient norm from $0.827$ to $0.055$, while preserving tangential
sensitivity ($6.117$ versus $6.217$).
Stronger isotropic penalties suppress both components, whereas
TASERLight preserves tangential sensitivity over a wider range
of tested strengths. Its largest weight also suppresses tangential
sensitivity, showing that excessive regularisation can compromise
this selectivity. This experiment is designed to demonstrate sensitivity suppression away from the data manifold. Further details and numerical results are in
Appendix~\ref{app:synthetic}.

\begin{figure}[t]
\centering
\includegraphics[width=0.95\linewidth]
    {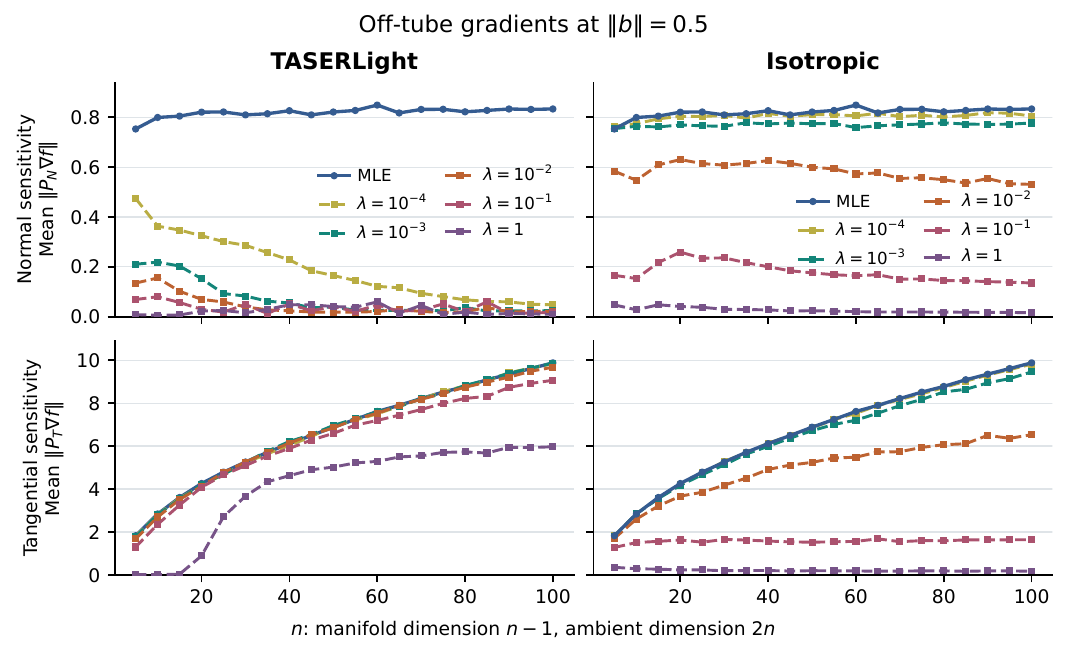}
\caption{Normal (top) and tangential (bottom) gradient sensitivity
for TASERLight (left) and isotropic regularisation (right),
with MLE in both columns; $\lambda$ denotes regularisation strength. Dimension varies at fixed
$\|b\|=0.5$: the manifold has dimension $n-1$, the ambient
dimension is $2n$, and the training sample size is $N=2048n$.}
\label{fig:toy_sphere_geometry}
\end{figure}

\subsection{TASER experiments}

\begin{table*}[t]
\centering
\scriptsize
\setlength{\tabcolsep}{3pt}
\renewcommand{\arraystretch}{1.12}
\caption{\footnotesize Clean and shifted-test performance (mean $\pm$ SE over three
model-training seeds, with auxiliary models held fixed).
Average rows report unweighted means across classification datasets.
Their SEs are propagated from the dataset-specific SEs assuming
independent training variability across datasets.
Bold marks the best displayed mean on each shifted-test row; implementation details
are provided in Appendix~\ref{app:regularisation}.}
\label{tab:regularisation-overview}
\begin{tabular}{@{}llccccccc@{}}
\toprule
Dataset & Test & MLE & Isotropic
& \shortstack{Layer\\spectral}
& TNAR
& TASERLight & TASER & \shortstack{Isotropic\\+ TASER} \\
\midrule
\multicolumn{9}{l}{\textit{Spectroscopy: standardised MSE $\downarrow$}} \\
Tablets & Source
& $0.056\pm0.001$ & $0.048\pm0.001$
& $0.050\pm0.001$
& $0.057\pm0.002$
& $0.078\pm0.003$
& $0.045\pm0.002$ & $0.044\pm0.001$ \\
& Instrument 2
& $0.072\pm0.002$ & $0.076\pm0.001$
& $0.081\pm0.005$
& $0.079\pm0.001$
& $0.069\pm0.001$
& $\mathbf{0.062}\pm0.001$ & $0.067\pm0.002$ \\
\midrule
\multicolumn{9}{l}{\textit{Classification: accuracy (\%) $\uparrow$}} \\
Skeleton & Clean
& $77.44\pm2.15$ & $86.53\pm1.18$
& $84.18\pm0.61$
& $86.87\pm1.46$
& $86.53\pm0.89$
& $86.53\pm0.44$ & $88.22\pm0.44$ \\
& Corrupted
& $77.27\pm1.83$ & $86.03\pm1.18$
& $84.12\pm0.62$
& $86.64\pm1.58$
& $85.86\pm0.98$
& $85.86\pm0.48$ & $\mathbf{87.99}\pm0.35$ \\
ECG & Clean
& $71.74\pm1.04$ & $79.10\pm0.64$
& $79.51\pm0.58$
& $70.98\pm2.83$
& $80.68\pm0.35$
& $79.71\pm1.20$ & $80.86\pm0.10$ \\
& Corrupted
& $71.13\pm0.71$ & $76.75\pm0.62$
& $77.66\pm0.61$
& $70.16\pm2.61$
& $79.24\pm0.45$
& $\mathbf{79.90}\pm1.26$ & $79.46\pm0.22$ \\
\addlinespace[2pt]
MNIST & Clean
& $99.08\pm0.00$ & $99.24\pm0.01$
& $99.07\pm0.01$
& $99.24\pm0.01$
& $99.34\pm0.02$
& $99.26\pm0.01$ & $99.21\pm0.01$ \\
& MNIST-C
& $78.84\pm0.20$ & $84.05\pm0.36$
& $81.17\pm0.38$
& $85.49\pm0.67$
& $85.00\pm0.63$
& $\mathbf{86.45}\pm0.20$ & $84.46\pm0.23$ \\
\addlinespace[2pt]
CIFAR-10 & Clean
& $89.30\pm0.13$ & $92.03\pm0.03$
& $92.72\pm0.04$
& $92.98\pm0.04$
& $92.14\pm0.08$
& $92.21\pm0.08$ & $91.68\pm0.07$ \\
& CIFAR-10-C
& $73.36\pm0.09$ & $81.39\pm0.05$
& $77.36\pm0.05$
& $81.05\pm0.42$
& $79.86\pm0.28$
& $81.78\pm0.31$ & $\mathbf{82.87}\pm0.19$ \\
\midrule
Average$^*$ & Clean
& $84.39\pm0.60$ & $89.23\pm0.33$
& $88.87\pm0.21$
& $87.52\pm0.79$
& $89.67\pm0.24$
& $89.43\pm0.32$ & $89.99\pm0.12$ \\
& Corrupted
& $75.15\pm0.49$ & $82.06\pm0.35$
& $80.08\pm0.24$
& $80.84\pm0.79$
& $82.49\pm0.32$
& $83.50\pm0.35$ & $\mathbf{83.70}\pm0.13$ \\
\bottomrule
\end{tabular}
\end{table*}

We test whether TASER improves predictive performance under distribution
shift by encouraging smoothness informed by the training distribution.
Across regression and classification tasks, we compare its effect with
distribution-agnostic regularisers and examine whether combining
score-directed and isotropic penalties provides additional benefit.
Clean and shifted test results are reported separately to assess both
robustness to shift and any cost to clean performance
(Table~\ref{tab:regularisation-overview}).
\paragraph{Evaluation tasks.}
We use SiLU MLPs for tablet assay regression from spectra
\citep{idrc2002shootout} and normal-versus-ventricular ECG beat
classification using the MIT-BIH Arrhythmia Database
\citep{moody2001impact}, with patient-disjoint splits for ECG.
Small SiLU CNNs classify MNIST \citep{lecun1998gradient} and
CIFAR-10 \citep{krizhevsky2009learning} images, using a residual
architecture for CIFAR-10 \citep{he2016deep}. Distribution shifts
comprise changes of measurement instrument for spectroscopy,
recorded ECG noise from the MIT-BIH Noise Stress Test Database
\citep{moody1984noise}, and the MNIST-C \citep{mu2019mnist} and
CIFAR-10-C \citep{hendrycks2019benchmarking} corruption benchmarks.
We report standardised MSE for spectroscopy, balanced accuracy
for ECG, and accuracy for image classification.

\paragraph{Score estimation.}
Scores are fitted to unlabelled training inputs independently of the
prediction task. We use analytical scores from regularised Gaussian
density estimates for spectroscopy and ECG, and noise-conditioned
neural networks trained by denoising score matching
\citep{vincent2011connection,song2019generative} for MNIST and
CIFAR-10. Score-model selection uses internal denoising holdouts;
implementation details are deferred to the appendix.

\paragraph{Regularisation and baselines.}
We compare TASER and its combination with isotropic
regularisation against MLE, isotropic input-gradient regularisation
\citep{jakubovitz2018improving}, layerwise spectral penalties
\citep{yoshida2017spectral}, and a supervised adaptation of
tangent-normal adversarial regularisation (TNAR)
\citep{yu2019tangent}. TNAR uses a variational autoencoder to
estimate tangent spaces and regularises prediction consistency
under tangent and normal perturbations. Regularisation strengths
and perturbation radii are selected using clean validation data.
Training, calibration, and aggregation details, are deferred
to the appendix.

\paragraph{Results.}
TASER achieves the lowest target-instrument MSE on tablets
($0.062$) and the highest corrupted-test accuracy on ECG
($79.90\%$) and MNIST-C ($86.45\%$) among the displayed methods
(Table~\ref{tab:regularisation-overview}). However, on MNIST-C, the combination reduces accuracy from
$86.45\%$ with TASER alone to $84.46\%$, showing that
composition is beneficial in some settings but does not
provide a uniformly better recipe.
On CIFAR-10, the combination improves corrupted-test accuracy
from $81.78\%$ to $82.87\%$ relative to TASER alone, while
reducing clean accuracy from $92.21\%$ to $91.68\%$,
illustrating a meaningful trade-off between clean and
shifted performance.
TASERLight achieves a slightly lower average corrupted accuracy
($82.49\%$) than TASER ($83.50\%$), while still exceeding
isotropic regularisation ($82.06\%$) and TNAR ($80.84\%$).
Combining TASER with isotropic regularisation achieves the
highest corrupted-test accuracy on skeleton ($87.99\%$) and
CIFAR-10 ($82.87\%$), yielding the highest average corrupted
accuracy ($83.70\%$).
These findings support the use of TASER for regularisation under
distribution shift and its composability with simpler penalties, and justify the use of TASERLight as a computationally cheaper variant with first-order smoothness preference. 

\subsection{
\gddr{A}dversarial robustness}
\label{sec:adversarial}

As a complementary evaluation, we study TASER on CIFAR-10 with
ResNet-18 \citep{he2016deep} under a controlled adversarial-training
protocol without data augmentation. Across PGD
\citep{madry2018towards}, TRADES \citep{zhang2019theoretically},
MART \citep{wang2020improving}, and their tested AWP combinations
\citep{wu2020adversarial}, adding TASER increases the mean of
AutoAttack \citep{croce2020reliable} and SPSA
\citep{uesato2018adversarial} accuracies by $6.32$ percentage points,
with a mean clean-accuracy change of $-0.49$ points. Both attacks
use an $\ell_\infty$ budget of $8/255$. Full results and
implementation details are reported in
Appendix~\ref{app:adversarial}.

These gains are sensitive to the training protocol. In a separate
TRADES ablation, the AutoAttack improvement decreases from $7.4$
points without augmentation to $1.3$ points with horizontal flips;
under the full augmentation pipeline, TASER reduces AutoAttack
and clean accuracy by $8.7$ and $12.2$ points, respectively (see Appendix~\ref{app:adversarial-augmentation}).
Thus, the current implementation does not transfer directly to
augmentation-rich adversarial training. A possible explanation is
mismatch between the score-model distribution and the augmented
inputs at which the penalty is evaluated, although the ablation
does not isolate this mechanism. These experiments provide
proof-of-concept evidence for TASER as a complementary regulariser
in a restricted adversarial setting; extending these gains to
stronger contemporary training recipes remains future work.

\subsection{Ablations and runtime analysis}

Additional experiments examine TASER's sensitivity to the score field
and its computational cost. In single-seed MNIST/MNIST-C ablations,
norm-preserving rotations of the estimated score reduce accuracy from
$86.43\%$ to $82.55$--$83.47\%$, while changing the noise-conditioning
level yields $79.28$--$84.79\%$, indicating sensitivity to both score
direction and smoothing scale (Table~\ref{tab:mnist-score-ablation}).
TASERLight avoids Laplacian computation and has training costs broadly
comparable to isotropic regularisation, whereas full TASER requires
$1.19$--$3.75$ times its per-epoch training time across datasets
(Table~\ref{tab:training-runtime-per-epoch}). These timings use cached
auxiliary quantities and exclude one-off preparation. Inference costs
remain comparable across methods because all evaluate the same
model architecture without auxiliary models or derivative
computations (Table~\ref{tab:inference-runtime-per-sample}).

\section{Discussion and Limitations}

\paragraph{A distribution-aware training principle.}
TASER \gddr{and TASERLight} provide compact and explicit mechanisms for incorporating
information about the training distribution into model optimisation.
They convert Stein-operator responses into a regularisation objective,
using distributional information to shape model sensitivity during
training. This extends their familiar role in distributional diagnostics
to actively influencing the learned model. Our controlled examples
demonstrate selective sensitivity suppression, while the real-data
experiments show improvements under distribution shift and useful
composition with isotropic regularisation.

\paragraph{Conditional theoretical guarantees.}
The analysis provides progressively stronger conclusions under additional
assumptions. The weighted Sobolev identity relates the full penalty to
model curvature and density-weighted gradients; strong log-concavity
makes the gradient contribution positive. The $\chi^2$-neighbourhood
bound controls the worst-case expected Stein response, while an additional
Poincar\'e inequality yields a bound on changes in the expected probe.
These are conditional population-level results: their regularity,
density, and shift assumptions may not hold for real datasets.
In particular, the exact-score, full-operator guarantees do not
automatically extend to our estimated-score nor do they directly guarantee classification accuracy or adversarial
robustness.

\paragraph{Implementation and computational cost.}
TASERLight requires a score estimate and input gradients; full TASER
additionally computes the Laplacian. Scores can come from simple density
models, dedicated networks or compatible public checkpoints, be cached
on fixed inputs, and be reused across model fits. Auxiliary models
are unnecessary at inference. In matched A10 timings, TNAR took
$2.4$ times as long as TASERLight on tablets, whereas TASERLight took
$3.0$ times as long as TNAR on MNIST. Full TASER cost $1.19$--$3.75$
times as much as TASERLight. These measurements exclude shared auxiliary
preparation; computational advantages depend on the task and implementation.

\paragraph{Score estimation and geometric validity.}
Accurate score estimation becomes more challenging as dimensionality
and distributional complexity increase, particularly with limited data.
Estimation error, smoothing, and mismatch between score-training and
penalty-input distributions can alter the imposed regularisation geometry.
Moreover, the score can reflect variation along the data manifold as
well as away from it. Suppressing score-aligned sensitivity therefore need not preserve
every task-relevant variation.

\paragraph{Scope of the empirical evidence.}
Across several domains, our results show that score-informed
regularisation can improve shifted-test performance and complement
isotropic smoothness. Gains nevertheless depend on the dataset,
architecture, strength and baseline. Fixed splits and auxiliary
estimators mean that model-seed variability captures only part of
the uncertainty, and scalability to larger datasets such as full
ImageNet remains untested.
The supplementary adversarial experiments provide a promising
\emph{proof of concept}, with robustness gains across several training
objectives under a controlled protocol without augmentation.
However, standard crop-and-flip augmentation reverses the benefit
in our TRADES ablation. These findings motivate further development
as a complement to adversarial training, while identifying augmentation
compatibility as a necessary step towards modern training pipelines.
TASER and TASERLight also require application-specific validation,
particularly in healthcare: our ECG experiments are methodological
demonstrations and do not establish clinical safety or effectiveness.

\paragraph{Probe choice and outlook.}
The probe determines which aspect of prediction is regularised.
Logits, margins, losses, and learned representations need not induce
equivalent behaviour, and principled probe selection remains open.
Further work should connect this choice to the prediction task and
develop tolerance to score error and augmentation mismatch.
TASER \gddr{and TASERLight} offer a modular interface between density estimation and
predictive learning, with theoretical motivation and empirical support
for their use alongside or on top of simpler regularisers.

\section*{AI use statement}
We used generative AI tools, including OpenAI Codex and Anthropic
Claude Code, to assist with implementation and debugging,
experimental design, data preparation, discussion of mathematical arguments, literature and reference checks,
and manuscript and figure preparation. These tools contributed
suggestions, code and draft text throughout the project. The authors
take full responsibility for the correctness, originality and
reporting of the work.

\bibliography{references}
\bibliographystyle{plainnat}

\clearpage
\appendix
\onecolumn
\thispagestyle{empty}
\aistatstitle{TASER: Task-Aware Stein Regularisation for Geometry-Driven Smoothness: \\
Supplementary Materials}

\DoToC

\section{Background on Stein's Method}
\label{app:stein-background}

This appendix gives a short introduction to Stein's method and the operator
viewpoint used in this paper. The goal is to provide enough background for a
reader unfamiliar with Stein's method to understand why Stein operators provide
distribution-dependent identities, how these identities are used in statistics,
and how they have been adapted in modern machine learning.

\subsection{Stein identities}

Stein's method is a general framework for characterising probability
distributions through expectation identities. Let $p$ be a target distribution
on a space $\mathcal X$. A \emph{Stein operator} for $p$ is an operator
$\mathcal T_p$ acting on a class of test functions $\mathcal F_p$ such that
\begin{equation}
\label{eq:app-stein-identity-general}
    \mathbb E_{X\sim p}\!\left[\mathcal T_p f(X)\right] = 0
    \qquad \text{for all } f\in \mathcal F_p .
\end{equation}
The class $\mathcal F_p$ is usually called a {\it Stein class}. The defining property
of a Stein operator is therefore that it produces functions with zero
expectation under the target distribution.

A classical example is the standard normal distribution. If $Z\sim
\mathcal N(0,1)$ and $f$ is sufficiently regular, then integration by parts gives
\begin{equation}
\label{eq:app-normal-stein}
    \mathbb E\!\left[f'(Z) - Z f(Z)\right] = 0 .
\end{equation}
Conversely, under appropriate conditions, if a random variable $W$ satisfies
\[
    \mathbb E\!\left[f'(W) - W f(W)\right] = 0
\]
for a sufficiently rich class of functions $f$, then $W$ has the standard
normal distribution. Thus the operator
\[
    \mathcal T_{\mathcal N} f(x) = f'(x) - x f(x)
\]
characterises the standard normal distribution.

The same idea extends far beyond the Gaussian case. For a differentiable
density $p$ on $\mathbb R^d$, \gddr{the Stein density operator given in \eqref{eq:app-first-order-stein}}, 
is
\begin{equation}
\label{eq:app-first-order-stein}
    \mathcal A_p \phi(x)
    =
    \nabla \cdot \phi(x) + \phi(x)^\top \nabla \log p(x),
\end{equation}
where $\phi:\mathbb R^d\to\mathbb R^d$ is a vector-valued test function. Under
appropriate boundary conditions,
\begin{equation}
\label{eq:app-first-order-identity}
    \mathbb E_{X\sim p}\!\left[\mathcal A_p \phi(X)\right] = 0 .
\end{equation}
Indeed,
\[
    \mathcal A_p \phi(x)
    =
    \frac{1}{p(x)} \nabla \cdot \left(p(x)\phi(x)\right),
\]
so that
\[
    \int p(x)\mathcal A_p\phi(x)\,dx
    =
    \int \nabla\cdot(p(x)\phi(x))\,dx,
\]
which vanishes when the boundary flux is zero.

\subsection{The Stein equation and distributional approximation}

In classical probability and statistics, Stein's method is 
often used to
bound distances between probability distributions. Suppose $p$ is a target
distribution and $q$ is another distribution. Given a test function $h$, one
constructs a solution $f_h$ to the \emph{Stein equation}
\begin{equation}
\label{eq:app-stein-equation}
    \mathcal T_p f_h(x)
    =
    h(x) - \mathbb E_{Z\sim p}[h(Z)] .
\end{equation}
If $X\sim q$, then taking expectations gives
\begin{equation}
\label{eq:app-stein-equation-expectation}
    \mathbb E_q[h(X)] - \mathbb E_p[h(Z)]
    =
    \mathbb E_q[\mathcal T_p f_h(X)] .
\end{equation}
Thus, a difference in expectations under $q$ and $p$ can be expressed as an
expectation of a Stein operator under $q$.

This is the basic mechanism behind many Stein bounds. If one can control
$\mathbb E_q[\mathcal T_p f_h(X)]$ uniformly over a class of test functions
$h$, then one obtains a bound on a probability metric between $q$ and $p$.
For example, by choosing different classes of $h$, one may obtain bounds in
Wasserstein distance, Kolmogorov distance, total variation distance, or other
integral probability metrics. Much of classical Stein theory is concerned with
constructing Stein equations for specific target distributions and proving
regularity estimates for their solutions.

This conventional use of Stein's method differs from the use in TASER. In the
classical setting, the test function $f_h$ is usually chosen by solving a Stein
equation associated with a discrepancy of interest. In TASER, by contrast, the
function is the model being trained. The Stein operator is therefore used
not primarily to compare two distributions, but to impose a distribution-aware
constraint on the model.
\subsection{The Langevin Stein operator}

The main text \gddr{also} uses the Langevin Stein operator $\mathcal L_p$. For a scalar function
$f:\mathbb R^d\to\mathbb R$ and a differentiable density $p$, 
\begin{equation}
\label{eq:app-langevin}
    \mathcal L_p f(x)
    =
    \Delta f(x) + \nabla \log p(x)^\top \nabla f(x),
\end{equation}
where $\Delta f=\mathrm{tr}(\nabla^2 f)$ is the Euclidean Laplacian, \gr{$\grad$ is the gradient, and the superscript $^\top$ denotes the transpose}.

The operator \eqref{eq:app-langevin} also has a divergence form:
\begin{equation}
\label{eq:app-langevin-divergence}
    \mathcal L_p f(x)
    =
    \frac{1}{p(x)}
    \nabla \cdot \left(p(x)\nabla f(x)\right).
\end{equation}
Consequently, under suitable regularity and boundary decay assumptions \gr{detailed in Appendix \ref{app:proofs}},
\begin{equation}
\label{eq:app-langevin-identity}
    \mathbb E_{X\sim p}\!\left[\mathcal L_p f(X)\right] = 0 .
\end{equation}

The Langevin operator is also the infinitesimal generator of the overdamped
Langevin diffusion
\begin{equation}
\label{eq:app-langevin-diffusion}
    dX_t = \nabla \log p(X_t)\,dt + \sqrt{2}\,dW_t,
\end{equation}
for which $p$ is an invariant distribution. In this interpretation,
$\mathcal L_p f(x)$ is the instantaneous expected rate of change of $f(X_t)$
when the diffusion starts from $x$. The identity
$\mathbb E_p[\mathcal L_p f]=0$ says that, at stationarity, this expected
instantaneous change averages to zero.

This generator viewpoint 
connects Stein identities to
diffusion geometry: $\nabla\log p$ describes the drift toward high-density
regions of the distribution, while $\Delta f$ captures isotropic second-order
variation of the test function.

\subsection{Stein discrepancies}


A related and very influential modern viewpoint is to use Stein operators to
define discrepancies between probability distributions. Let $p$ be the target
density and $q$ a candidate distribution. For a Stein operator $\mathcal T_p$
and a function class $\mathcal F$, define
\begin{equation}
\label{eq:app-stein-discrepancy}
    \mathcal S(q,p)
    =
    \sup_{f\in\mathcal F}
    \left|
    \mathbb E_{X\sim q}[\mathcal T_p f(X)]
    \right|.
\end{equation}
Since $\mathbb E_p[\mathcal T_p f]=0$ for all $f\in\mathcal F$, the quantity
$\mathcal S(q,p)$ measures how strongly samples from $q$ violate Stein
identities that hold under $p$.

Stein discrepancies are attractive because they often require only the score
$\nabla\log p(x)$ rather than the normalised density $p(x)$. This is important
in Bayesian statistics and probabilistic modelling, where the target density is
often known only up to an unknown normalising constant. Since $\nabla \log p(x)$
is invariant to multiplication of $p$ by a constant, \gr{Langevin} Stein operators can be
computed even when the normalising constant is unavailable.

A particularly important example is the \emph{Kernel Stein Discrepancy} (KSD)
\citep{liu2016kernelized,gorham2015measuring}.
KSD takes the Stein test functions to lie in a reproducing kernel Hilbert space
(RKHS), which allows the supremum in \eqref{eq:app-stein-discrepancy} to be
computed in closed form. For the first-order Langevin Stein operator, the KSD
can be written as
\begin{equation}
\label{eq:app-ksd}
    \mathrm{KSD}^2(q,p)
    =
    \mathbb E_{X,X'\sim q}
    \left[
        k_p(X,X')
    \right],
\end{equation}
where $k_p$ is a Stein kernel obtained by applying the Stein operator to both
arguments of a base kernel $k$. In one common form,
\begin{align}
\label{eq:app-stein-kernel}
k_p(x,x')
&=
s_p(x)^\top k(x,x')s_p(x')
+
s_p(x)^\top \nabla_{x'} k(x,x') 
+
s_p(x')^\top \nabla_x k(x,x')
+
\mathrm{tr}\!\left(\nabla_x\nabla_{x'} k(x,x')\right),
\end{align}
where $s_p(x)=\nabla\log p(x)$.

KSD has been used for goodness-of-fit testing, measuring sample quality,
diagnosing Markov chain Monte Carlo, and variational inference. Its appeal is
that it yields a computable discrepancy from samples of $q$ and score evaluations
of $p$, without requiring samples from $p$ or knowledge of the normalising
constant.

\subsection{\gr{Langevin} Stein operators in machine learning}

Stein methods have entered machine learning through several routes. First,
Stein discrepancies provide practical objectives and diagnostics for
probabilistic modelling. They have been used to assess whether generated or
sampled particles match a target distribution, to build goodness-of-fit tests,
and to train approximate inference procedures.

Second, the score function $\nabla\log p(x)$ \gr{appearing in the Langevin Stein operator \eqref{eq:stein-operator}} has become a central object in
modern generative modelling. Score matching and diffusion models learn vector
fields approximating the score of noisy data distributions. Since \gr{Langevin} Stein
operators are built from score functions, they provide a natural mathematical
interface between score-based generative models and downstream learning
objectives.

Third, \gr{Langevin} Stein operators provide a way to incorporate distributional geometry into
learning. The terms in $\mathcal L_p f=\Delta f+s_p^\top\nabla f$ combine
curvature of the learned function with directional derivatives along the score
field of the data distribution. Thus, unlike standard isotropic regularisers
such as weight decay or gradient penalties, Stein-based regularisation can
adapt to the geometry of the input distribution.

The present work follows this third direction. Rather than using Stein
operators to construct a goodness-of-fit test or a discrepancy over a large
function class, TASER applies the \gr{Langevin} Stein operator directly to the model being
trained. The resulting penalty encourages the model to have controlled
Stein residuals under the training distribution. In this sense, TASER adapts
Stein's method from a tool for distribution comparison into a mechanism for
distribution-aware regularisation.

\section{Proofs} \label{app:proofs}

For the theoretical derivation of the results in the paper, we first define the Stein class $\mathcal{F}(p)$ \gr{for the Langevin Stein operator $\mathcal{L}_p$ in \eqref{eq:stein-operator},} as \gr{for example} in \cite{kozyra2026tastetaskawareoutofdistributiondetection}.
\begin{definition}[Stein class for $\mathcal{L}_p$]
\label{def:stein-class-app}
Let $p$ be a continuously differentiable density on $\R^d$. A function
$f:\R^d\to\R$ belongs to the Stein class of $p$ (for $\mathcal{L}_p$), denoted
$f\in\mathcal{F}(p)$, if:
\begin{enumerate}[label=\textnormal{(S\arabic*)}, leftmargin=*]
\item \label{S1}
$f$ is twice continuously differentiable and $\Delta f$, $\nabla f$ are locally
integrable with respect to Lebesgue measure.
\item \label{S2}
The vector field $p(x)\nabla f(x)$ is integrable and its flux over spheres
vanishes:
\[
\lim_{R\to\infty}\int_{\partial B_R} p(x)\nabla f(x)\cdot n(x)\,dS(x)=0,
\]
where $B_R \subset \R^d$ is the Euclidean ball of radius $R$ in $\R^d$, and $n(x)$ is the outward unit normal.
\item \label{S3}
$\mathcal{L}_p f$ is integrable under $p$.
\end{enumerate}
\end{definition}

In \cite{kozyra2026tastetaskawareoutofdistributiondetection} it is shown that 
if $p$ is continuously differentiable and $f\in\mathcal{F}(p)$, then \gr{the Stein identity \eqref{eq:app-langevin-identity} holds, namely}
$
  \E_{X \sim p} [\mathcal{L}_p f(X)] = 0.
$ 

\subsection{Relationship between the operators}
\label{app:taserlight}

For the identities below, assume that $p$ is positive on its open
support, the displayed derivatives exist, all expectations are
finite, and boundary terms vanish in each integration by parts used.
Recall from \eqref{eq:first-order} that
\[
\mathcal B_p f
=s_p^\top\nabla f-\mathbb E_p[s_p^\top\nabla f].
\]
The alternative construction in \eqref{eq:stein-operator-light}
follows from
\[
\frac{\Delta p}{p}
=\frac{\operatorname{div}(p\,s_p)}{p}
=\operatorname{div}s_p+\|s_p\|^2
\]
and integration by parts:
\[
\mathbb E_p\!\left[f\frac{\Delta p}{p}\right]
=\int f\Delta p
=-\int\nabla f^\top\nabla p
=-\mathbb E_p[s_p^\top\nabla f].
\]

\begin{proposition}\label{prop:lp}
Suppose that $f\in\mathcal F(p)$ and that $\Delta f$ and
$s_p^\top\nabla f$ are integrable under $p$. Then
\[
\mathcal L_p f
=\mathcal B_p f+\Delta f-\mathbb E_p[\Delta f].
\]
\end{proposition}
\begin{proof}
Since $\mathbb E_p[\mathcal L_p f]=0$,
\[
\mathbb E_p[s_p^\top\nabla f]
=-\mathbb E_p[\Delta f].
\]
Hence $\mathcal B_p f=s_p^\top\nabla f+\mathbb E_p[\Delta f]$.
Substituting this into
$\mathcal L_p f=\Delta f+s_p^\top\nabla f$ proves the assertion
and \eqref{eq:lpconnection}.
\end{proof}

The operators therefore differ by the fluctuations of the
Laplacian around its expectation under $p$. Both ignore constant
additive offsets: $\mathcal L_p(f+c)=\mathcal L_p f$ and
$\mathcal B_p(f+c)=\mathcal B_p f$.

\subsection{Proof of \eqref{eq:sobolev-stein-identity}}

We clarify the assumptions for \eqref{eq:sobolev-stein-identity} in the following result. 

\begin{proposition} Let $\mathcal{L}_{p}$ be as in \eqref{eq:stein-operator}. 
 Assume that $p$ is a twice continuously differentiable strictly positive  probability density,  that $f\in\mathcal{F}(p)$,   that $\mathcal{L}_{p}f$ 
is integrable as well as differentiable, and that $f \in C^3(\R^d)$.
Assume further that ${\mathcal{L}}_{p}f \in L^2(p)$, that $\nabla f^\top \nabla ({\mathcal{L}}_{p}f)  \in L^1(p),$  assume that all expectations below are finite, and that the integration-by-parts identity  \[
\mathbb E_p[u\,\mathcal L_p v]
=
-\mathbb E_p[\nabla u^\top\nabla v],
\] holds for $(u,v)= ({\mathcal L}_p f, f)$ and for $(u,v)=(\partial_i f, \partial_i f), i=1, \ldots, d.$
Let 
$
H_p(x) = -\nabla^2 \log p(x)
$. Then 
\begin{equation*}
\mathbb{E}_{p}\!\left[
(\mathcal{L}_p f(X))^2
\right]
=
\mathbb{E}_{p}\!\left[
\|\nabla^2 f(X)\|_F^2
\right]
+
\mathbb{E}_{p}\!\left[
\nabla f(X)^\top H_p(X)\nabla f(X)
\right].
\end{equation*}
\end{proposition}
\begin{proof}
Write $s_p=\nabla\log p$. Differentiating the operator gives
\[
\nabla(\mathcal L_p f)
=
\mathcal L_p(\nabla f)
+
(\nabla^2\log p)\nabla f,
\]
where $\mathcal L_p$ acts componentwise on $\nabla f$. This expression is well-defined as $f \in C^3(\R^d).$
Using  the assumed integration-by-parts identity, 
we obtain
\begin{align*}
\mathbb E_p[(\mathcal L_p f)^2]
&=
-\mathbb E_p[
  \nabla f^\top\nabla(\mathcal L_p f)]=
-\sum_{i=1}^d
  \mathbb E_p[
    (\partial_i f)\,\mathcal L_p(\partial_i f)]
-
\mathbb E_p[
  \nabla f^\top(\nabla^2\log p)\nabla f].
\end{align*}
Applying integration by parts to each term in the sum yields
\[
-\sum_{i=1}^d
\mathbb E_p[
  (\partial_i f)\,\mathcal L_p(\partial_i f)]
=
\sum_{i=1}^d
\mathbb E_p[\|\nabla(\partial_i f)\|^2]
=
\mathbb E_p[\|\nabla^2 f\|_F^2].
\]
Substituting $H_p=-\nabla^2\log p$, which is well defined for strictly positive $p$,  proves
\[
\mathbb E_p[(\mathcal L_p f)^2]
=
\mathbb E_p[\|\nabla^2 f\|_F^2]
+
\mathbb E_p[\nabla f^\top H_p\nabla f].
\]
\end{proof}

\paragraph{What changes for TASERLight.}
The corresponding first-order penalty is
\[
\mathcal R_{\mathcal B}(f)
=\mathbb E_p[\nabla f^\top s_ps_p^\top\nabla f]
 -\bigl(\mathbb E_p[s_p^\top\nabla f]\bigr)^2.
\]
It controls fluctuations in the score-directed derivative,
with a local weighting matrix $s_ps_p^\top$ of rank at most one.
It does not generally control the full gradient or Hessian,
so the second-order Sobolev identity above does not extend to
TASERLight. The annulus example in Section~\ref{app:annulus}
illustrates this distinction.

\subsection{Shift bounds for the operators}

Let $q$ be absolutely continuous with respect to $p$ and write
$\ell=q/p$. Recall that
\[
\chi^2(q\|p)=\mathbb E_p[(\ell-1)^2],
\qquad
\mathcal Q_\rho=\{q\ll p:\chi^2(q\|p)\leq\rho^2\}.
\]
Here $\mathcal A_p$ denotes either scalar operator
$\mathcal L_p$ or $\mathcal B_p$.

\begin{proposition}\label{prop:operator-shift}
Suppose that $\mathcal A_p f\in L^2(p)$ and
$\mathbb E_p[\mathcal A_p f]=0$. Then
\[
\sup_{q\in\mathcal Q_\rho}
\left|\mathbb E_q[\mathcal A_p f]\right|
\leq\rho\sqrt{\mathbb E_p[(\mathcal A_p f)^2]}.
\]
\end{proposition}
\begin{proof}
Changing measure and using the zero-mean identity gives
\[
\mathbb E_q[\mathcal A_p f]
=\mathbb E_p[\ell\mathcal A_p f]
=\mathbb E_p[(\ell-1)\mathcal A_p f].
\]
Cauchy--Schwarz therefore yields
\[
\left|\mathbb E_q[\mathcal A_p f]\right|
\leq\sqrt{\chi^2(q\|p)}
      \sqrt{\mathbb E_p[(\mathcal A_p f)^2]}.
\]
Taking the supremum over $\mathcal Q_\rho$ proves
\eqref{eq:sup-shift-bound} for both operators.
\end{proof}

\paragraph{Geometric interpretation.}
The two operators control different average responses to a shift.
For full TASER, assume additionally that $q$ is positive and
continuously differentiable on the common support, that the
following expectations are finite, and that the boundary flux
of $q\nabla f$ vanishes. Integration by parts under $q$ gives
\[
\mathbb E_q[\Delta f]
=-\mathbb E_q[(\nabla\log q)^\top\nabla f].
\]
Consequently,
\begin{align*}
\mathbb E_q[\mathcal L_p f]
&=\mathbb E_q[(\nabla\log p-\nabla\log q)^\top\nabla f]=-\mathbb E_q\!\left[
\nabla f^\top\nabla\log\frac{q}{p}\right],
\end{align*}
which proves \eqref{eq:projection-identity-shift}.
The vector $\nabla\log(q/p)$ points towards increasing density
under $q$ relative to $p$. Its inner product with $\nabla f$
measures the probe's directional response to that change.
Thus the full-operator bound controls the average signed
alignment of probe sensitivity with the relative-density gradient.

For TASERLight, centring instead gives the direct interpretation
\[
\mathbb E_q[\mathcal B_p f]
=\mathbb E_q[s_p^\top\nabla f]
 -\mathbb E_p[s_p^\top\nabla f].
\]
Its bound controls how much the expected score-directed response
changes between the training and shifted distributions.
By Proposition~\ref{prop:lp}, these interpretations are related by
\[
\mathbb E_q[\mathcal B_p f]
=\mathbb E_q[\mathcal L_p f]
 -\bigl(\mathbb E_q[\Delta f]-\mathbb E_p[\Delta f]\bigr).
\]
The shift in the expected Laplacian is therefore the correction
needed to relate TASERLight to the full operator's alignment
interpretation. The operator bound itself requires neither
smoothness of $q$ nor these additional integration-by-parts
conditions.

\subsection{Poincar\'e bounds on expected probe shifts}

Suppose additionally that $p$ satisfies a Poincar\'e inequality with spectral
 gap $\kappa>0$:
\[
\operatorname{Var}_p(g)
\leq\kappa^{-1}\mathbb E_p[\|\nabla g\|^2]
\]
for functions in its Sobolev domain.
Let $f$ be in the domain of the Langevin generator, with
$f,\mathcal L_p f\in L^2(p)$, and assume that the integration
by parts below is valid. Then
\begin{equation}\label{eq:poincare}
\sup_{q\in\mathcal Q_\rho}
\left|\mathbb E_q[f]-\mathbb E_p[f]\right|
\leq\frac{\rho}{\kappa}
\sqrt{\mathbb E_p[(\mathcal L_p f)^2]}.
\end{equation}
To see this, similarly as in the proof above, first apply Cauchy--Schwarz:
\begin{align*}
\left|\mathbb E_q[f]-\mathbb E_p[f]\right|
&=\left|\mathbb E_p[(\ell-1)(f-\mathbb E_p f)]\right| \leq\sqrt{\chi^2(q\|p)}\sqrt{\operatorname{Var}_p(f)}
\leq\rho\sqrt{\operatorname{Var}_p(f)}.
\end{align*}

Self-adjointness of the Langevin generator and the assumption that the boundary flux vanishes give 
\[
  \mathbb E_p\lVert\nabla f\rVert^2
  =-\mathbb E_p[(f-\mathbb E_p f)\mathcal{L}_p  f]
  \le \sqrt{Var_p(f)}\sqrt{\mathbb E_p[(\mathcal{L}_p f(X))^2] }
\]
where we applied the Cauchy-Schwarz inequality in the last step. 
Combining this with $\kappa Var_p(f)\le \mathbb E_p\lVert\nabla f\rVert^2$ yields
\[
 \sqrt{ {Var_p(f)}} \le 
  \frac{1}{\kappa}\sqrt{\mathbb E_p[(\mathcal{L}_p f(X))^2] }
\] 
and \eqref{eq:poincare} follows.

For TASERLight, suppose additionally that the conditions of
Proposition~\ref{prop:lp} hold and that
$\mathcal B_p f,\Delta f\in L^2(p)$.
The operator decomposition and the triangle inequality give
\[
\sqrt{\mathbb E_p[(\mathcal L_p f)^2]}
\leq\sqrt{\mathbb E_p[(\mathcal B_p f)^2]}
    +\sqrt{\operatorname{Var}_p(\Delta f)}.
\]
Substituting this into \eqref{eq:poincare} yields
\begin{equation}\label{eq:poincarelight}
\sup_{q\in\mathcal Q_\rho}
\left|\mathbb E_q[f]-\mathbb E_p[f]\right|
\leq\frac{\rho}{\kappa}
\left(
\sqrt{\mathbb E_p[(\mathcal B_p f)^2]}
+\sqrt{\operatorname{Var}_p(\Delta f)}
\right).
\end{equation}
Thus TASERLight gives a bound on the expected probe shift when
fluctuations of the Laplacian are also controlled.

\subsection{An annulus example}
\label{app:annulus}

\begin{example}
Let $p$ be a smooth, rotationally symmetric density on an annulus
$a<\|x\|<b$ in $\mathbb R^2$, where $0<a<b<\infty$, with
boundary decay sufficient for the preceding integrations by parts.
In polar coordinates, the score $s_p$ is radial, and the angle
$\Theta$ is uniform and independent of the radius $R$.
Suppose that the task depends on the angle but not the radius,
with probe $f(r,\vartheta)=\cos\vartheta$.
Its gradient is tangential, so
\[
s_p^\top\nabla f=0,
\qquad \mathcal B_p f=0.
\]
TASERLight therefore leaves this angular variation unpenalised.

For the full operator, the polar-coordinate Laplacian gives
\[
\Delta f
=\frac{\partial^2 f}{\partial r^2}
 +\frac1r\frac{\partial f}{\partial r}
 +\frac1{r^2}\frac{\partial^2 f}{\partial\vartheta^2}
=-\frac{\cos\vartheta}{r^2}.
\]
Consequently,
\[
\mathcal R_{\mathcal B}(f)=0,
\qquad
\mathcal R_{\mathcal L}(f)
=\mathbb E_p\!\left[\frac{\cos^2\Theta}{R^4}\right]
=\frac12\mathbb E_p[R^{-4}]>0.
\]
Full TASER also penalises the curvature of the angular signal.
This example illustrates why the choice between the two
operators depends on the geometry relevant to the task.
\end{example}

\section{Synthetic data: experimental details}
\label{app:synthetic}

Both examples use analytic scores, separating the effect of
regularisation from score-estimation error. The scalar probe is
the regression output.

\paragraph{One-dimensional regression.}
Inputs follow $x\sim\mathcal N(0,1)$ with noiseless targets
$y=\sin(x)$. We train a three-hidden-layer tanh MLP of width
$128$ using Adam for $5{,}000$ updates, with learning rate
$2\times10^{-3}$ and fresh batches of $256$ samples.
The MSE objective is supplemented by either
$\lambda\|\theta\|_2^2$ or the full TASER penalty, using
$\mathcal L_p f=f''-xf'$ and
$\lambda\in\{0,10^{-4},10^{-3},10^{-2},10^{-1}\}$.
Derivatives are computed exactly by automatic differentiation;
the squared residual is used without empirical centring.
We evaluate the final models on $20{,}000$ independent
$\mathcal N(0,1)$ samples and a uniform grid of $2{,}000$
points in $[-10,10]$, using $\sin(x)$ as the target throughout.
The shaded interval in the plots contains the central $95\%$
of the training distribution.

\paragraph{Hypersphere: directional sensitivity.}
We generate $x=(a,b)\in\mathbb R^{2n}$ according to
\[
a=u+0.1\xi,\qquad b=0.01\eta,
\]
where $u$ is uniform on $S^{n-1}$ and $\xi,\eta$ are
independent standard Gaussian vectors. The target is
$y=\sum_{j=1}^{n}\sin(a_j)$.
For $n\in\{5,10,\ldots,100\}$, we use $2048n$ training
samples and a two-hidden-layer SiLU MLP of width $256$.
Training uses Adam for $40$ epochs, with learning rate
$10^{-3}$ and batch size $256$.
We compare MSE training alone with centred TASERLight,
using the analytic score of this noisy-sphere density and
$\lambda=10^{-3}$.

For evaluation, we retain the training marginal of $a$ and
replace $b$ by a random direction of norm
$\delta\in\{0.1,0.5,1,2\}$.
We average Euclidean gradient-component norms over $1024n$
evaluation points. With $q=a/\|a\|$, the normal component
contains the radial projection of $\nabla_a f$ onto $q$
and all of $\nabla_b f$; the tangent component is the
remaining part of $\nabla_a f$.
Thus, $\delta$ measures displacement in the extra coordinates,
rather than total distance to the sphere.
These diagnostics assess directional sensitivity, not shifted
prediction error: the target itself varies radially.

Both demonstrations use one fit per configuration, with
configuration-specific seeds; they do not estimate
across-seed uncertainty.

\begin{table}[H]
\centering
\small
\setlength{\tabcolsep}{4pt}
\caption{Hypersphere experiment ($n=40$, ambient dimension $80$):
mean Euclidean norms of normal and tangential gradient components
at extra-coordinate displacement $\delta=\|b\|$.
Normal components include radial sensitivity in $a$ and all
sensitivity in $b$. Results use one fit per method.}
\label{tab:toy_sphere_geometry}
\begin{tabular}{@{}ccccc@{}}
\toprule
& \multicolumn{2}{c}{Normal}
& \multicolumn{2}{c}{Tangential} \\
\cmidrule(lr){2-3}\cmidrule(lr){4-5}
$\delta$ & MLE & TASERLight & MLE & TASERLight \\
\midrule
$0.1$ & 0.826 & 0.050 & 6.117 & 6.233 \\
$0.5$ & 0.827 & 0.055 & 6.117 & 6.217 \\
$1.0$ & 0.827 & 0.069 & 6.114 & 6.164 \\
$2.0$ & 0.828 & 0.113 & 6.104 & 5.966 \\
\bottomrule
\end{tabular}
\end{table}

\section{Distribution-shift regularisation: experimental details}
\label{app:regularisation}
\label{app:headline-experiments}

This section describes the experiments in
Table~\ref{tab:regularisation-overview}. We evaluate supervised prediction
on clean inputs and prescribed distribution shifts. The models are trained from
their initial weights with the regulariser present throughout training;
these results are not obtained by fine-tuning baseline checkpoints.
The adversarial-training experiments use a separate protocol described in
Appendix~\ref{app:adversarial}.

\paragraph{Asset licences.}
The MIT-BIH Arrhythmia and Noise Stress Test databases are
distributed under the Open Data Commons Attribution Licence
v1.0. MNIST is listed under CC BY-SA 3.0 in its Keras
documentation; the MNIST-C and CIFAR-10-C releases are
licensed under CC BY 4.0. For the IDRC tablet data, UTKinect,
and original CIFAR-10 data, we could not identify an explicit
dataset licence on the public distribution pages. We cite
their providers and do not redistribute these datasets.
The Gaussian density estimators, neural score models, and
VAE geometry models used in the headline experiments are
fitted within our experimental pipeline.

\subsection{Datasets, preprocessing, and distribution shifts}

\paragraph{Tablet spectroscopy.}
We use the pharmaceutical-tablet data from the 2002 IDRC Shootout
\citep{idrc2002shootout}: $655$ physical samples, $650$ spectral channels,
and measurements from two instruments. We preserve the supplied split of
$155$ training, $40$ validation, and $460$ test samples, predicting assay
from the third reference column. Instrument~1 supplies the training and
validation inputs. The two reported test conditions are instrument~1 and
instrument~2 measurements of the same held-out physical samples. Spectra
are centred by training-channel means and divided by one training RMS
scale; targets are standardised using their training mean and standard
deviation. Both test instruments use these source-training statistics.
The reported MSE is in standardised target units. Target-instrument
measurements do not fit the model or density, or select hyperparameters.

\paragraph{Skeleton action recognition.}
UTKinect contains ten action classes represented by $20$ three-dimensional
joints \citep{xia2012view}. After excluding one clip with a missing action
annotation, we use $199$ clips. Subjects $1,3,5,7$ provide $80$ training
clips, subject~$9$ provides $20$ validation clips, and subjects $2,4,6,8,10$
provide $99$ test clips. This is a subject-disjoint split with a reserved
validation subject, rather than the full five-subject training split.
We retain the first occurrence of duplicate frame identifiers, interpolate
each clip to $32$ frames, subtract the root-joint coordinates, and remove
the resulting three constant coordinates. Inputs therefore have shape
$32\times57$ and use a single training RMS scale, without per-clip rotation
or scale normalisation. The headline table uses the width-$8$ temporal CNN.
The corrupted result averages six fixed conditions: independent joint
jitter with standard deviations $5,10,20,40$\,mm and temporally smooth
coordinate offsets at $10,20$\,mm. These are simulated measurement errors
added to real root-relative sequences, with action labels unchanged.

\paragraph{ECG beat classification.}
We use the MLII channel of the MIT-BIH Arrhythmia Database
\citep{moody2001impact}, excluding paced records $102,104,107,217$ and
keeping records $201$ and $202$, which belong to the same patient, together.
The resulting $43$ patient groups are split into $25$ training, $8$
validation, and $10$ test groups. The task distinguishes normal-family
annotations $(N,L,R,e,j)$ from ventricular-family annotations $(V,E)$;
other beat types are excluded. Windows extend from $0.25$ seconds before
to $0.45$ seconds after the annotated beat and are resampled to $128$
samples. Thus, the task assumes a known beat location and does not evaluate
beat detection. Training is capped at $400$ beats per record and class,
giving $13{,}294$ training beats; validation and test retain $17{,}232$ and
$22{,}173$ eligible beats, respectively. Training-feature means and a
single training RMS scale are used without per-beat normalisation.
We add baseline-wander, muscle-artifact, and electrode-motion recordings
from the MIT-BIH Noise Stress Test Database \citep{moody1984noise} at
$24,12,6$\,dB, producing nine corrupted conditions. These are synthetic
mixtures of recorded ECG and recorded noise. Performance is balanced
accuracy, computed as the mean of the two class recalls, rather than
ordinary accuracy or an average over patients.

\paragraph{MNIST.}
We split the official $60{,}000$ development images into $55{,}000$ training
and $5{,}000$ validation examples using a fixed permutation with seed
$20260929$, and retain the official $10{,}000$ test images
\citep{lecun1998gradient}. Pixels are mapped to $[0,1]$, centred by the
training-pixel means, and divided by one training RMS scale. Training uses
no image augmentation. MNIST-C accuracy is the unweighted mean over its
$15$ non-identity corruptions \citep{mu2019mnist}; the clean test is reported
separately. All corruptions use the same preprocessing and test labels.

\paragraph{CIFAR-10.}
We split the official training set into $45{,}000$ training and $5{,}000$
validation images using a fixed PyTorch permutation with seed~$0$, and
retain the official $10{,}000$ test images \citep{krizhevsky2009learning}.
Pixels are mapped to $[-1,1]$. The headline table reports classifier
training with reflection padding of four pixels, random $32\times32$
crops, and random horizontal flips. The supervised loss uses the augmented
view, while TASER, isotropic, and TNAR input penalties use the corresponding
\emph{unaugmented} training image. This clean-view protocol differs from
the augmentation ablation in Appendix~\ref{app:adversarial-augmentation}.
CIFAR-10-C accuracy averages the $15$ corruption types and five severities
equally \citep{hendrycks2019benchmarking}. Corrupted examples are never
used for model training or hyperparameter selection.

\subsection{Models and optimisation}

Table~\ref{tab:app-headline-training} gives the model settings for
the headline comparison. All hidden activations are SiLU. The skeleton
CNN has two kernel-$5$ temporal convolutions and adaptive average pooling
into four temporal bins. The MNIST CNN has two kernel-$5$ convolutions,
each followed by average pooling, and a $64$-unit hidden layer. The
CIFAR-10 network has a $32$-channel stem, two residual blocks at each
width, eight-group GroupNorm, and global average pooling. It is distinct
from the ResNet-18 used in the adversarial experiments.

\paragraph{The probe.} \mrk{We use the natural predicted-class probability as the probe. Wherever the winning class is separated from its competitors, this probe coincides locally with a smooth class-probability function, allowing its input derivatives to be evaluated by differentiating the selected branch. Nevertheless, class-switching boundaries prevent direct application of our global smooth-probe identities without additional boundary analysis. In the robustness experiments, we use a smooth logit-based probe to obtain a globally smooth sensitivity measure across perturbation neighbourhoods, including regions where the predicted class may change.} 

\paragraph{Laplacian estimation.}
For full TASER runs, we estimate the Laplacian using
Hutchinson's estimator \citep{hutchinson1989stochastic}:
\[
\widehat{\Delta f_\theta}(x)
=\frac{1}{K}\sum_{k=1}^{K}
\xi_k^\top \nabla_x^2 f_\theta(x)\,\xi_k,
\]
where the probe vectors are independent and have independent
Rademacher entries ($\pm1$ with equal probability).
We use $K=4$ probes.

Although this estimator is unbiased for $\Delta f_\theta(x)$,
squaring the estimated Stein residual introduces an additional
variance term. Writing
\[
\widehat{r}_\theta(x)
=\widehat{\Delta f_\theta}(x)
+s_p(x)^\top\nabla f_\theta(x),
\]
we obtain, with expectation taken over the probe vectors,
\[
\mathbb{E}_{\xi}\!\left[\widehat{r}_\theta(x)^2\right]
=
\bigl(\mathcal{L}_p f_\theta(x)\bigr)^2
+\frac{2}{K}\sum_{i\neq j}
\bigl(\partial_{ij}f_\theta(x)\bigr)^2.
\]
Thus, the stochastic squared-residual objective includes
an additional nonnegative penalty on mixed second derivatives,
whose strength decreases as $1/K$. With independent probes
across examples, minibatch centring multiplies the averaged
variance contribution by $1-1/B$, where $B$ is the batch size,
but does not remove it.

\subsection{Baselines.}

\paragraph{Isotropic regularisation.}
We penalise the squared Frobenius norm of the model's input
Jacobian, averaged over outputs, using regression predictions or
class-centred logits \citep{jakubovitz2018improving}.
This suppresses input sensitivity uniformly across directions,
without using an estimated data density.

\paragraph{Spectral regularisation.}
We penalise the sum of squared largest singular values of layer
weight matrices, excluding biases \citep{yoshida2017spectral}.
Convolutional kernels are flattened into output-channel by
input-kernel matrices, with singular values estimated by power
iteration. This controls sensitivity through layer weights,
without using data-density estimates.

\paragraph{TNAR.}
Our supervised adaptation of tangent-normal adversarial
regularisation \citep{yu2019tangent} estimates tangent spaces
from the decoder Jacobian of a VAE fitted to unlabelled training
inputs \citep{kingma2014auto}. It averages prediction-consistency
penalties under tangent and normal perturbations, found using
two power iterations within the estimated subspaces.
Consistency uses KL divergence for classification and squared
prediction change for regression. The fitted geometry is frozen
and needed only during training.

Writing $q$ for training-input RMS per coordinate and $d$ for
input dimension, tangent and normal perturbations use the same
$\ell_2$ radius $cq\sqrt d$, with $c\in\{0.03,0.1,0.3\}$.
We cross these radii with the stated regularisation-strength
grid and select using clean validation data. The finite-difference
radius is $0.001q\sqrt d$, with two power iterations and equal
weights on the tangent and normal penalties.

\paragraph{Combined penalties.}
For isotropic regularisation combined with TASER or TASERLight,
we retain the coefficients selected independently for the two
standalone penalties and train a new predictor from initialization.
The combined fit selects only its checkpoint using clean validation
data; the two coefficients are not jointly retuned.

\begin{table*}[t]
\centering\small
\setlength{\tabcolsep}{5pt}
\caption{Models and training budgets for the headline comparison.
Batch sizes refer to the supervised loss.}
\label{tab:app-headline-training}
\begin{tabular}{@{}llllrr@{}}
\toprule
Dataset & Model & Optimiser & Initial learning rate & Epochs & Batch \\
\midrule
Tablets & MLP ($128,64$) & Adam & $10^{-3}$ & $200$ & $32$ \\
Skeleton & Temporal CNN ($8,8$) & Adam & $10^{-3}$ & $200$ & $32$ \\
ECG & MLP ($64,32$) & Adam & $10^{-3}$ & $80$ & $256$ \\
MNIST & CNN ($16,32$), hidden width $64$ & Adam & $10^{-3}$ & $80$ & $256$ \\
CIFAR-10 & Residual CNN ($32,64,128$) & SGD & $0.1$ & $100$ & $128$ \\
\bottomrule
\end{tabular}
\end{table*}

\subsection{Score estimation}

\paragraph{Gaussian scores.}
For tablets and skeletons, we fit a Gaussian with a PCA covariance in
a retained subspace and positive isotropic variance in its orthogonal
complement. The residual variance is the mean discarded variance times
$a\in\{1,4,16\}$, with floor $10^{-5}$; retained variances are bounded
below by this residual variance. Candidate ranks are
$r\in\{2,5,10,20,40\}$. Selection minimises mean normalised Gaussian
denoising error at noise scales $0.01,0.03,0.1$, using four fixed noise
draws on an internal training holdout. This holds out approximately
$20\%$ of tablet training samples or skeleton training subject~$7$.
Both selected models use rank $40$ and multiplier $4$, and are refitted
on all model-training inputs. Scores are evaluated analytically
without additional smoothing, retaining both PCA and residual components.
The skeleton density is fitted to pooled frames and provides a
frame-marginal score, rather than a joint density over whole sequences.

For ECG, a Gaussian covariance is shrunk towards its mean eigenvalue
times the identity. The coefficient is selected from
$\{0.001,0.01,0.1,0.5,1\}$ by denoising error on a $10\%$ internal
training holdout. Eigenvalues have floor $10^{-5}$, and adding
$0.1^2I$ gives the covariance used for the smoothed score.
This estimator retains its fit on the remaining $90\%$ of training
beats; it is not refitted on the density holdout.

\paragraph{Neural scores.}
For MNIST and CIFAR-10, noise-conditioned U-Nets are trained from scratch
by denoising score matching
\citep{vincent2011connection,song2019generative}. The network predicts
the Gaussian noise added at a log-uniform scale $\sigma\in[0.01,1]$;
the score estimate is minus the predicted noise divided by $\sigma$.
Networks use channel widths $32,64,128$, GroupNorm, SiLU, and a log-noise
embedding. Training uses Adam at $2\times10^{-4}$, batch size $256$,
a $500$-update warmup, and exponential moving average decay $0.995$.

MNIST uses an internal $44{,}000/11{,}000$ fit/holdout split and a
$60$-epoch selection budget; CIFAR-10 uses $40{,}000/5{,}000$ and
$200$ epochs. Duration is selected using average denoising error at
$\sigma\in\{0.01,0.03,0.1\}$ on $2{,}000$ fixed holdout images.
The selected duration is then refitted from the same initialisation on
all model-training images. Estimator seed $901$ is fixed across
model runs. MNIST score fitting uses unaugmented images; the headline
CIFAR-10 score model uses crop-and-flip augmentation.

Unless varied explicitly, penalties query the estimator at $\sigma=0.03$.
These scores target Gaussian-smoothed densities. Scores are cached on
training inputs and detached from model optimisation. Neither labels
nor outer validation or test inputs enter density fitting or selection.

\subsection{Calibration and baseline tuning}

\paragraph{Penalty calibration.}
Since the magnitude of the penalty term for each method is potentially different, for each dataset and model seed, TASERLight, full TASER,
isotropic, layerwise spectral and TNAR penalties are scaled as
$\lambda=\rho G_{\mathrm{task}}/G_R$, where $G_{\mathrm{task}}$
and $G_R$ are RMS parameter-gradient norms at initialisation,
estimated over eight fixed training batches.
Thus, $\rho$ controls the initial regulariser-to-task gradient ratio.
Candidate ratios are selected using clean validation data, and the
resulting coefficients remain fixed throughout training. The search grid is $\rho\in\{10^{-3},10^{-2},10^{-1},1\}$ outside
CIFAR-10 and $\{10^{-4},10^{-3},10^{-2}\}$ on CIFAR-10 (Full TASER on CIFAR-10 instead uses
the smaller grid $\{10^{-3},10^{-2}\}$ because of its higher
computational cost).

\paragraph{Validation}
Each model is trained for its full allocated epoch budget.
The checkpoint and penalty strength are selected using
clean validation performance.

\section{Adversarial training: experimental details}
\label{app:experimental-details}
\label{app:adversarial}

These experiments supplement our main evaluation of regularisation under
distribution shift. We examine whether TASER can improve adversarial robustness
in a controlled CIFAR-10 setting, and then test the sensitivity of this benefit
to data augmentation. The comparison across training objectives uses no data
augmentation; a separate TRADES ablation introduces horizontal flips and the
full crop-and-flip pipeline.

\paragraph{Asset licences.}
The pretrained \texttt{google/ddpm-cifar10-32} checkpoint
is listed under Apache~2.0, as is the Diffusers library used
to load it. The official AutoAttack implementation is
distributed under the MIT licence. CIFAR-10 is obtained
from its original distribution site, which does not specify
an explicit dataset licence. Neither the dataset nor the
pretrained checkpoint is redistributed in our reproduction
package.

\subsection{Data and architecture}

We use the standard CIFAR-10 split of $50{,}000$ training and $10{,}000$ test
images, with ten classes and resolution $32\times32$. Inputs are normalised
using the dataset mean and standard deviation. The classifier is ResNet-18
\citep{he2016deep}, adapted to CIFAR-10 with an initial $3\times3$, stride-one
convolution and no initial max-pooling. Perturbation budgets are expressed in the original
pixel scale $[0,1]$.

\subsection{Training objectives}

For each objective, we train a baseline and a corresponding model with TASER
included from the start of training. The TASER models in this comparison are
not obtained by fine-tuning baseline checkpoints. Both variants use the same
base training procedure.

\paragraph{Standard and PGD training.}
Standard training minimises cross-entropy. PGD training instead minimises
cross-entropy on adversarial examples generated within an $\ell_\infty$ ball
of radius $8/255$. The training attack uses ten projected gradient steps of
size $2/255$.

\paragraph{TRADES and MART.}
TRADES \citep{zhang2019theoretically} combines clean cross-entropy with a KL
divergence between clean and adversarial predictive distributions, with
coefficient $\beta=6$. Its inner attack maximises this KL divergence.
MART \citep{wang2020improving} uses its misclassification-aware adversarial
objective with robustness coefficient $5$. Both use ten attack steps, step
size $2/255$, and $\ell_\infty$ radius $8/255$.

\paragraph{Adversarial weight perturbation.}
We also evaluate TRADES and MART with adversarial weight perturbation
(AWP; \citealp{wu2020adversarial}). Weights are temporarily perturbed when
computing the robust loss and restored before the parameter update. The
configuration uses \texttt{awp\_gamma=0.005}, \texttt{awp\_rho=5e-3},
\texttt{awp\_num\_steps=1}, and \texttt{awp\_start\_epoch=10}.

\paragraph{TASER and optimisation.}
For a base objective $\mathcal L_{\mathrm{base}}$, the regularised objective is
\begin{equation}
\label{eq:app-taser-training-objective}
\mathcal L_{\mathrm{total}}(\theta)
=\mathcal L_{\mathrm{base}}(\theta)
+\lambda\,\mathcal R_{\mathrm{TASER}}(\theta).
\end{equation}
The base objective retains its own adversarial-example construction. TASER is
evaluated on the corresponding non-adversarial input. In the augmentation
ablation this is the augmented image, before adversarial perturbation;
``non-adversarial'' therefore does not imply ``unaugmented.''

Models are trained for $200$ epochs with batch size $256$, using Adam with
initial learning rate $10^{-3}$, parameters $(0.9,0.999)$, weight decay
$10^{-4}$, and cosine learning-rate decay. The TASER coefficient is
$\lambda=1$ unless otherwise stated.

\paragraph{Smooth logit-margin probe.}
For logits $\ell_{\theta,k}(x)$ and target class $y$, we use
\[
f_\theta(x,y)
=\log\!\sum_{k\neq y}\exp\!\bigl(\ell_{\theta,k}(x)\bigr)
-\ell_{\theta,y}(x).
\]
This replaces the largest competing logit with a smooth
log-sum-exp approximation at unit temperature. Lower values
indicate stronger separation of the target class from its
competitors, linking probe sensitivity to classification margins.
The label is held fixed when differentiating with respect to
the input. Reversing the probe's sign leaves both squared
regularisation penalties unchanged.

\subsection{Score model}

We use a fixed pretrained unconditional diffusion score model for CIFAR-10 \citep{ho2020denoising}, available as
\texttt{google/ddpm-cifar10-32}, evaluated at
timestep $t=50$. For a noise-prediction parameterisation, its output is
converted to a score through
\[
\widehat s(x_t,t)=-\epsilon_\phi(x_t,t)/\sigma_t,
\]
where $\sigma_t$ is the noise standard deviation under the checkpoint's
diffusion convention. The score model is held fixed across all experiments. In particular, the augmentation
ablation changes the distribution of inputs on which the penalty is evaluated
without adapting the score model to each augmented distribution.

\subsection{Adversarial evaluation}

We evaluate AutoAttack \citep{croce2020reliable} using the standard version of
the official implementation, and SPSA \citep{uesato2018adversarial} using
$32$ iterations and $128$ samples per iteration. Both reported attacks use
$\ell_\infty$ radius $8/255$. AutoAttack is evaluated on the full test set
unless otherwise stated; the more expensive query-based evaluation uses a
fixed subset of $1{,}000$ test images shared across methods.

Table~\ref{tab:taser_tradeoff} preserves the reported aggregate robustness
values and their bootstrap intervals from the original comparison.
Table~\ref{tab:taser_attack_breakdown} gives the available attack-specific
breakdown. An aggregate across attacks is a descriptive summary and should
not be interpreted as accuracy against an adversary that can choose between
the attacks separately for each image.

\subsection{Computational cost}
\label{app:compute}

Each adversarial-training run uses one NVIDIA A10 GPU. The reported training
time ratios are $2.45$ for standard training and $1.17$--$1.27$ for the
adversarial-training objectives (Table~\ref{tab:taser_tradeoff}). These timings
describe the evaluated architecture and training recipes; they do not establish
the cost of extending TASER to larger or more elaborate pipelines.

\subsection{Controlled comparison without augmentation}

Table~\ref{tab:taser_tradeoff} reports the effect of TASER across the six
training objectives. In this setting, every objective shows a positive change
in the reported robustness aggregate. Across the five adversarial-training
objectives, the average gain is $6.32$ percentage points, accompanied by a
$0.49$-point reduction in clean accuracy. The reported bootstrap interval for
this average clean-accuracy change includes zero; this does not establish
equivalence of clean performance. The standard classifier exhibits a larger
robustness gain and a larger clean-accuracy cost.

\begin{table*}[t]
\centering
\caption{Controlled CIFAR-10 comparison without augmentation (ResNet-18).
Accuracies are percentages; changes are percentage points. Brackets contain
the reported $95\%$ bootstrap intervals for changes. Robustness aggregates
summarise AutoAttack and SPSA at $\ell_\infty$ radius $8/255$.
Overhead is the training-time ratio with versus without TASER.}
\label{tab:taser_tradeoff}
\setlength{\tabcolsep}{4pt}
\resizebox{\linewidth}{!}{%
\begin{tabular}{l cc r cc r c}
\toprule
& \multicolumn{3}{c}{Clean accuracy}
& \multicolumn{3}{c}{Reported robust accuracy}
& Overhead \\
\cmidrule(lr){2-4}\cmidrule(lr){5-7}
Method & Baseline & +TASER & $\Delta$ [95\% CI]
& Baseline & +TASER & $\Delta$ [95\% CI] & \\
\midrule
Vanilla
& 77.88 & 70.92 & $-6.96$ {\scriptsize$[-8.15,-5.75]$}
& 3.00 & 19.25 & $+16.25$ {\scriptsize$[14.40,18.10]$}
& $\times2.45$ \\
PGD
& 66.03 & 65.64 & $-0.39$ {\scriptsize$[-1.68,0.94]$}
& 24.25 & 33.35 & $+9.10$ {\scriptsize$[6.35,11.90]$}
& $\times1.25$ \\
TRADES
& 67.07 & 65.61 & $-1.46$ {\scriptsize$[-2.78,-0.15]$}
& 27.35 & 33.95 & $+6.60$ {\scriptsize$[3.70,9.40]$}
& $\times1.27$ \\
MART
& 64.07 & 63.31 & $-0.76$ {\scriptsize$[-2.09,0.57]$}
& 26.35 & 35.00 & $+8.65$ {\scriptsize$[5.85,11.50]$}
& $\times1.19$ \\
TRADES + AWP
& 67.90 & 68.87 & $+0.97$ {\scriptsize$[-0.31,2.24]$}
& 34.05 & 36.20 & $+2.15$ {\scriptsize$[-0.80,5.10]$}
& $\times1.21$ \\
MART + AWP
& 66.96 & 66.14 & $-0.82$ {\scriptsize$[-2.14,0.45]$}
& 32.30 & 37.40 & $+5.10$ {\scriptsize$[2.20,8.05]$}
& $\times1.17$ \\
\midrule
Average
& 68.32 & 66.75 & $-1.57$ {\scriptsize$[-2.10,-1.04]$}
& 24.55 & 32.53 & $+7.98$ {\scriptsize$[6.87,9.09]$} & -- \\
Average (excl. vanilla)
& 66.41 & 65.91 & $-0.49$ {\scriptsize$[-1.09,0.09]$}
& 28.86 & 35.18 & $+6.32$ {\scriptsize$[5.06,7.60]$} & -- \\
\bottomrule
\end{tabular}%
}
\end{table*}

\paragraph{Attack-specific results.}
The available breakdown in Table~\ref{tab:taser_attack_breakdown} shows
improvements under both AutoAttack and SPSA for the unaugmented training
protocol. These findings support a benefit within this controlled setting;
they do not establish transfer to other attack budgets or stronger training
recipes.

\begin{table}[htbp]
\centering
\caption{Clean and robust accuracy on CIFAR-10 (ResNet-18), with runtime overhead from TASER.}
\label{tab:taser_attack_breakdown}
\setlength{\tabcolsep}{3.5pt}
\resizebox{\linewidth}{!}{%
\begin{tabular}{l
                ccc c
                ccc c c}
\toprule
& \multicolumn{4}{c}{\textbf{No TASER}}
& \multicolumn{5}{c}{\textbf{With TASER}} \\
\cmidrule(lr){2-5} \cmidrule(lr){6-10}
Method
& Clean
& \shortstack{Robust\\(AA)}
& \shortstack{Robust\\(SPSA)}
& \shortstack{Runtime\\(ms/sample)}
& Clean
& \shortstack{Robust\\(AA)}
& \shortstack{Robust\\(SPSA)}
& \shortstack{Runtime\\(ms/sample)}
& \shortstack{Relative\\overhead} \\
\midrule

Vanilla
& 77.88 & 0.00 & 6.00 & 0.0756
& 70.92 & 10.60 & 27.90 & 0.1852
& $\times 2.45$ \\

PGD
& 66.03 & 17.80 & 30.70 & 0.3714
& 65.64 & 26.20 & 40.50 & 0.4652
& $\times 1.25$\\

TRADES
& 67.07 & 18.90 & 35.80 & 0.4118
& 65.61 & 27.70 & 40.20 & 0.5236
& $\times 1.27$\\

TRADES + AWP
& 67.90 & 27.30 & 40.80 & 0.4794
& 68.87 & 29.90 & 42.50 & 0.5795
& $\times 1.21$\\

MART
& 64.07 & 18.40 & 34.30 & 0.4062
& 63.31 & 28.30 & 41.70 & 0.4861
& $\times 1.19$\\

MART + AWP
& 66.96 & 25.70 & 38.90 & 0.4739
& 66.14 & 31.20 & 43.60 & 0.5555
& $\times 1.17$\\

\bottomrule
\end{tabular}%
}
\end{table}

\subsection{Interaction with data augmentation}
\label{app:adversarial-augmentation}

We separately repeat TRADES training while varying only the augmentation
protocol: no augmentation, random horizontal flips, or the full pipeline of
horizontal flips and padded random crops. Each protocol is evaluated with and
without TASER. These are separate runs from the preceding comparison, so
changes are computed within each matched pair in
Table~\ref{tab:taser_augmentation}.

\begin{table*}[t]
\centering
\caption{TRADES augmentation ablation on CIFAR-10 (ResNet-18).
Accuracies are percentages; changes are percentage points relative to the
baseline using the same augmentation. Both attacks use $\ell_\infty$ radius
$8/255$. Full augmentation comprises random horizontal flips and padded
random crops.}
\label{tab:taser_augmentation}
\setlength{\tabcolsep}{6pt}
\begin{tabular}{ll rrr rrr}
\toprule
& & \multicolumn{3}{c}{Accuracy}
& \multicolumn{3}{c}{Change with TASER} \\
\cmidrule(lr){3-5}\cmidrule(lr){6-8}
Augmentation & TASER & Clean & AA & SPSA
& $\Delta$ Clean & $\Delta$ AA & $\Delta$ SPSA \\
\midrule
None  & No  & 67.2 & 22.0 & 33.0 & -- & -- & -- \\
      & Yes & 65.9 & 29.4 & 40.0 & $-1.3$ & $+7.4$ & $+7.0$ \\
\addlinespace
Flips & No  & 72.3 & 28.7 & 42.0 & -- & -- & -- \\
      & Yes & 65.4 & 30.0 & 48.0 & $-6.9$ & $+1.3$ & $+6.0$ \\
\addlinespace
Full  & No  & 74.1 & 34.4 & 49.0 & -- & -- & -- \\
      & Yes & 61.9 & 25.7 & 38.0 & $-12.2$ & $-8.7$ & $-11.0$ \\
\bottomrule
\end{tabular}
\end{table*}

Stronger augmentation improves the underlying TRADES baseline: AutoAttack
accuracy increases from $22.0\%$ without augmentation to $34.4\%$ with the
full pipeline. At the same time, the incremental benefit of TASER deteriorates.
Without augmentation, TASER improves AutoAttack accuracy by $7.4$ points at
a clean-accuracy cost of $1.3$ points. With flips alone, the respective changes
are $+1.3$ and $-6.9$ points. Under full augmentation, TASER reduces
AutoAttack accuracy by $8.7$ points and clean accuracy by $12.2$ points.
SPSA exhibits the same progression from positive to negative robustness gains.

\paragraph{Interpretation and limitations.}
TASER derives its regularisation geometry from an estimated score, whose
relevance depends on the input distribution at which it is evaluated.
Augmentation changes this effective distribution. Applying a fixed score
model directly to transformed inputs can therefore impose an inappropriate
regularisation direction, even when those inputs retain their semantic labels.
The observed degradation is consistent with this score-mismatch explanation,
but the ablation does not isolate it from interactions involving optimisation
or regularisation strength.

The unaugmented results establish a proof of concept for combining TASER with
adversarial-training objectives in a restricted setting. The augmentation
ablation also demonstrates that the present implementation does not transfer
directly to the standard augmented pipeline. Extending these gains to modern
adversarial-training recipes requires further work on score estimation for the
effective training distribution and its interaction with regularisation and
optimisation. We therefore treat adversarial robustness as a supplementary
application; the main empirical evaluation concerns regularisation under
distribution shift.

\section{Score-field ablations}
\label{app:ablations}

\begin{table}[t]
\centering
\small
\setlength{\tabcolsep}{5pt}
\caption{MNIST-C accuracy (\%) under score-field ablations:
one seed, fixed regularisation strength. The original score uses
$\sigma=0.03$.}
\label{tab:mnist-score-ablation}
\begin{tabular}{@{}lccccccc@{}}
\toprule
Method & Original & $30^\circ$ & $60^\circ$ & $90^\circ$
& $\sigma=0.01$ & $\sigma=0.10$ & $\sigma=0.30$ \\
\midrule
TASER & 86.43 & 83.47 & 82.79 & 82.55
& 84.79 & 79.28 & 81.96 \\
\bottomrule
\end{tabular}
\end{table}

We test dependence on score geometry and noise level by perturbing each score vector towards a fixed random orthogonal
direction, producing angles of $30^\circ$, $60^\circ$, and $90^\circ$
while preserving its norm. Specifically, For each input $x$, let $s(x)$ be the estimated score and draw
a random vector $r(x)$, held fixed throughout training and
shared across angles. Define
\[
u(x)=\frac{s(x)}{\|s(x)\|},
\qquad
v(x)=
\frac{r(x)-\langle r(x),u(x)\rangle u(x)}
{\|r(x)-\langle r(x),u(x)\rangle u(x)\|}.
\]
Thus $v(x)$ is a unit vector orthogonal to the original score.
The perturbed field is
\[
s_\theta(x)
=
\cos(\theta)\,s(x)
+\sin(\theta)\,\|s(x)\|\,v(x),
\qquad
\theta\in\{30^\circ,60^\circ,90^\circ\}.
\]
By construction,
\[
\|s_\theta(x)\|=\|s(x)\|,
\qquad
\frac{\langle s_\theta(x),s(x)\rangle}
{\|s_\theta(x)\|\,\|s(x)\|}
=\cos(\theta).
\]

The angles are measured relative to the estimated score,
not the unknown true score. 

We also evaluate the same frozen score
network at noise levels $\sigma\in\{0.01,0.10,0.30\}$, relative to
the original $\sigma=0.03$. These change the network's noise
conditioning; all scores are evaluated on clean inputs.
Each variant retrains the model with the original penalty
coefficient and clean-validation checkpoint selection.
The noise-level comparison changes the target smoothing scale
and should not be interpreted as a monotonic degradation in score quality.

\section{Computational cost}
\label{app:cost}
Tables~\ref{tab:training-runtime-per-epoch}
and~\ref{tab:inference-runtime-per-sample} report training and
inference times on NVIDIA A10 GPUs.
With architectures, optimisers and regularisation sample counts
matched within each dataset, full TASER requires approximately
$1.19$--$3.75$ times the training time of TASERLight.
TASERLight avoids Laplacian computation and has broadly similar
cost to isotropic regularisation. It is faster than our TNAR
adaptation on tablets, skeletons and ECG, while TNAR is faster
on MNIST and CIFAR-10.
Training timings reuse cached scores and tangent bases and
exclude one-off auxiliary preparation.
At inference, all four methods evaluate the same model
architecture without auxiliary models or derivative computations,
giving comparable per-sample costs.
We report mean $\pm$ two timing standard deviations across three
training repetitions and ten inference blocks, averaging within
each repetition or block before computing variability.
These measure timing variation rather than variation across
model-training seeds.

\begin{table*}[t]
\centering
\small
\setlength{\tabcolsep}{4pt}
\caption{Training time per epoch in seconds
(mean $\pm$ two timing SDs across three repetitions)
on NVIDIA A10 GPUs.
Warmup, validation, checkpointing and auxiliary preparation
are excluded. Full/Lite ratios are calculated within each
repetition.}
\label{tab:training-runtime-per-epoch}
\begin{tabular}{@{}lrrrrr@{}}
\toprule
Dataset & Isotropic & TNAR & TASERLight & Full TASER
& Full / Lite \\
\midrule
Tablets
& $0.02542\pm0.00004$
& $0.06109\pm0.00106$
& $0.02502\pm0.00063$
& $0.03165\pm0.00053$
& $(1.27\pm0.01)\times$ \\
Skeleton
& $0.02079\pm0.00009$
& $0.03896\pm0.00031$
& $0.02051\pm0.00075$
& $0.04979\pm0.00206$
& $(2.43\pm0.08)\times$ \\
ECG
& $0.3785\pm0.0093$
& $0.7047\pm0.0024$
& $0.3312\pm0.0034$
& $0.3950\pm0.0060$
& $(1.19\pm0.03)\times$ \\
MNIST
& $9.82\pm0.11$
& $3.35\pm0.02$
& $10.02\pm0.13$
& $22.95\pm0.12$
& $(2.29\pm0.04)\times$ \\
CIFAR-10
& $40.59\pm0.14$
& $17.36\pm0.12$
& $41.56\pm1.33$
& $155.82\pm2.18$
& $(3.75\pm0.11)\times$ \\
\bottomrule
\end{tabular}
\end{table*}

\begin{table*}[t]
\centering
\small
\setlength{\tabcolsep}{6pt}
\caption{Inference time in $\mu$s per sample
(mean $\pm$ two timing SDs across ten blocks) on an NVIDIA A10.
Each block averages 100 warmed-up, synchronised model
forward passes at batch size 256; batch latency is divided
by 256. Inputs are GPU-resident; preprocessing and transfers
are excluded.}
\label{tab:inference-runtime-per-sample}
\begin{tabular}{@{}lrrrr@{}}
\toprule
Dataset & Isotropic & TNAR & TASERLight & Full TASER \\
\midrule
Tablets
& $1.44\pm0.02$ & $1.45\pm0.04$
& $1.48\pm0.27$ & $1.44\pm0.01$ \\
Skeleton
& $1.22\pm0.01$ & $1.215\pm0.003$
& $1.25\pm0.21$ & $1.216\pm0.003$ \\
ECG
& $1.44\pm0.01$ & $1.44\pm0.01$
& $1.45\pm0.07$ & $1.44\pm0.01$ \\
MNIST
& $2.15\pm0.08$ & $2.13\pm0.02$
& $2.14\pm0.03$ & $2.14\pm0.02$ \\
CIFAR-10
& $57.96\pm0.84$ & $58.15\pm0.83$
& $57.98\pm0.88$ & $58.00\pm0.84$ \\
\bottomrule
\end{tabular}
\end{table*}
\end{document}